\documentclass{article}

\usepackage[final]{neurips_2025}

\usepackage[utf8]{inputenc} 
\usepackage[T1]{fontenc}    
\usepackage[colorlinks=true, linkcolor=purple, citecolor=green!60!black]{hyperref}       
\usepackage{url}            
\usepackage{booktabs}       
\usepackage{amsfonts}       
\usepackage{nicefrac}       
\usepackage{microtype}      
\usepackage{xcolor}         

\usepackage{tikz}
\usepackage{graphicx}
\usepackage{caption}
\usepackage{subcaption}
\usepackage{pifont}
\usepackage{algorithm}
\usepackage{algorithmic}

\usepackage{amsmath}
\usepackage{amssymb}
\usepackage{mathtools}
\usepackage{amsthm}
\usepackage{thmtools}
\usepackage{thm-restate}

\usepackage{tabularx}
\usepackage{enumitem}

\usepackage{multirow}
\usepackage{adjustbox}
\usepackage{wrapfig}

\makeatletter
\newcommand\footnoteref[1]{\protected@xdef\@thefnmark{\ref{#1}}\@footnotemark}
\makeatother

\usepackage[capitalize]{cleveref}
 \AtBeginDocument{%
    \crefname{figure}{Figure}{Figures}%
}
\newcommand{\eqlabelleft}{(}
\newcommand{\eqlabelright}{)}

\creflabelformat{equation}{\eqlabelleft#2#1#3\eqlabelright}

\theoremstyle{plain}

\theoremstyle{definition}

\theoremstyle{remark}

\newcommand{\appropto}{\mathrel{\vcenter{
  \offinterlineskip\halign{\hfil$##$\cr
    \propto\cr\noalign{\kern2pt}\sim\cr\noalign{\kern-2pt}}}}}

\newcommand{\cmark}{\ding{51}}  
\newcommand{\xmark}{\ding{55}}  

\usepackage{amsmath,amsfonts,bm}

\def\eqref#1{equation~\ref{#1}}

\def\1{\bm{1}}

\def\rvc{{\mathbf{c}}}

\def\rvs{{\mathbf{s}}}

\def\rvx{{\mathbf{x}}}

\def\rmI{{\mathbf{I}}}

\DeclareMathAlphabet{\mathsfit}{\encodingdefault}{\sfdefault}{m}{sl}
\SetMathAlphabet{\mathsfit}{bold}{\encodingdefault}{\sfdefault}{bx}{n}

\title{Training-Free Safe Text Embedding Guidance for \\ Text-to-Image Diffusion Models}

\author{
Byeonghu Na\textsuperscript{\textmd{1}} \quad Mina Kang\textsuperscript{\textmd{1}} \quad Jiseok Kwak\textsuperscript{\textmd{1}} \quad Minsang Park\textsuperscript{\textmd{1}} \\
\textbf{Jiwoo Shin\textsuperscript{\textmd{1}} \quad SeJoon Jun\textsuperscript{\textmd{1}} \quad Gayoung Lee\textsuperscript{\textmd{2}} \quad Jin-Hwa Kim\textsuperscript{\textmd{2,3}} \quad Il-Chul Moon\textsuperscript{\textmd{1,4}}} \\
\textsuperscript{\textmd{1}}KAIST,~ \textsuperscript{\textmd{2}}NAVER AI Lab,~ \textsuperscript{\textmd{3}}SNU AIIS,~ \textsuperscript{\textmd{4}}summary.ai\\
\texttt{\{byeonghu.na,kasong13,jskwak,pagemu,natu33,sjmathy,icmoon\}@kaist.ac.kr},\\
\texttt{\{gayoung.lee,j1nhwa.kim\}@navercorp.com}  \\
}

\begin{document}

\maketitle

\begin{abstract}
    Text-to-image models have recently made significant advances in generating realistic and semantically coherent images, driven by advanced diffusion models and large-scale web-crawled datasets. However, these datasets often contain inappropriate or biased content, raising concerns about the generation of harmful outputs when provided with malicious text prompts. We propose Safe Text embedding Guidance (STG), a training-free approach to improve the safety of diffusion models by guiding the text embeddings during sampling. STG adjusts the text embeddings based on a safety function evaluated on the expected final denoised image, allowing the model to generate safer outputs without additional training. Theoretically, we show that STG aligns the underlying model distribution with safety constraints, thereby achieving safer outputs while minimally affecting generation quality. Experiments on various safety scenarios, including nudity, violence, and artist-style removal, show that STG consistently outperforms both training-based and training-free baselines in removing unsafe content while preserving the core semantic intent of input prompts. Our code is available at \url{https://github.com/aailab-kaist/STG}.\\[0.5em]
    \color{red} Warning: This paper contains model-generated content that may be disturbing. \color{black}
\end{abstract}
 
\section{Introduction}
\label{sec:intro}

Recent advances in text-to-image models have received considerable attention for their ability to generate realistic images that semantically align with given text prompts~\cite{podell2024sdxl,ramesh2021zero,saharia2022photorealistic,xie2025sana}. These advances have largely been driven by the development of diffusion models~\cite{dhariwal2021diffusion,rombach2022high} and the availability of large-scale datasets collected through web crawling~\cite{changpinyo2021conceptual,schuhmann2022laion}. However, this approach to data collection often includes inappropriate or biased content, raising the risk that text-to-image models may generate images containing unsafe concepts, such as sexual content, violence, bias, or copyright infringement~\cite{birhane2023into,birhane2021multimodal}. Additionally, the concept of \textit{safe} can be defined in commercial settings, e.g., avoiding any intellectual property violations embodied by a certain style. Moreover, what is considered \textit{safe} can vary widely depending on individual sensitivities, cultural contexts, and social norms, making it challenging to define a universally safe model~\cite{naous2024beer,tao2024cultural}. This highlights the need for safe generation methods that can adapt to diverse perspectives and account for individual perceptions.

To address these challenges, several safe generation methods for diffusion models have been proposed in recent years. Concept unlearning approaches fine-tune the weights of the diffusion model to forget the unsafe content~\cite{gandikota2023erasing,liu2024latent,lyu2024one,parkdirect,zhang2024forget}. While this can be an effective way to remove unsafe concepts, it also presents a challenge in maintaining the original generative capabilities of the model. In addition, these methods require carefully curated safety-annotated text-image datasets for training, along with significant computational resources, which can limit their adaptability to diverse perspectives.  

Instead, training-free approaches for safe generation have also been proposed. Among them, filtering-based methods exclude unsafe concepts directly from the input prompt, preserving the original generative capacity~\cite{yang2024guardti}. However, these methods can be vulnerable to adversarial attacks, where carefully crafted prompts can bypass the filters, reducing their effectiveness. Alternatively, recent methods attempt to suppress unsafe content during the diffusion sampling process. These approaches include directly manipulating latent representations~\cite{schramowski2023safe}, and adjusting text embeddings or attention weights of the diffusion model~\cite{gandikota2024unified,gong2024reliable,yoon2024safree}. However, these approaches typically do not directly incorporate the intermediate or final samples produced by the diffusion model into their safety mechanisms (as illustrated in the third panel of \cref{fig:overview}), making it unclear how their safety methods influence the samples from the diffusion model. Furthermore, they often lack a clear theoretical foundation for understanding how their modifications affect the original model distribution.

In this paper, we propose Safe Text embedding Guidance (STG), a training-free approach for safe text-to-image diffusion models by guiding text embeddings in safer directions during the diffusion sampling process, as illustrated in the final panel of \cref{fig:overview}. STG is motivated by the observation that unsafe images often arise from text prompts that contain explicit or implicit unsafe concepts. Therefore, STG adjusts the text embeddings during the sampling process by directly incorporating intermediate latent samples for guidance, enabling the model to generate safer outputs without any additional training.
Specifically, we apply a safety function, originally evaluated on clean images, to the expected final denoised image from the current noisy image to obtain a safe guidance signal. Theoretically, we show that STG adjusts the perturbed data to align with the underlying model distribution and the desired safety constraints, generating safer outputs while minimizing degradation of the original generative quality. We experimentally validate STG for various safety scenarios, including nudity, violence, and artist-style removal. Our results show that STG consistently outperforms both training-based and training-free baselines, effectively removing unsafe content while preserving the original semantic intent of the input prompts.

\begin{figure}[tp]
    \centering
    \includegraphics[width=0.96\linewidth]{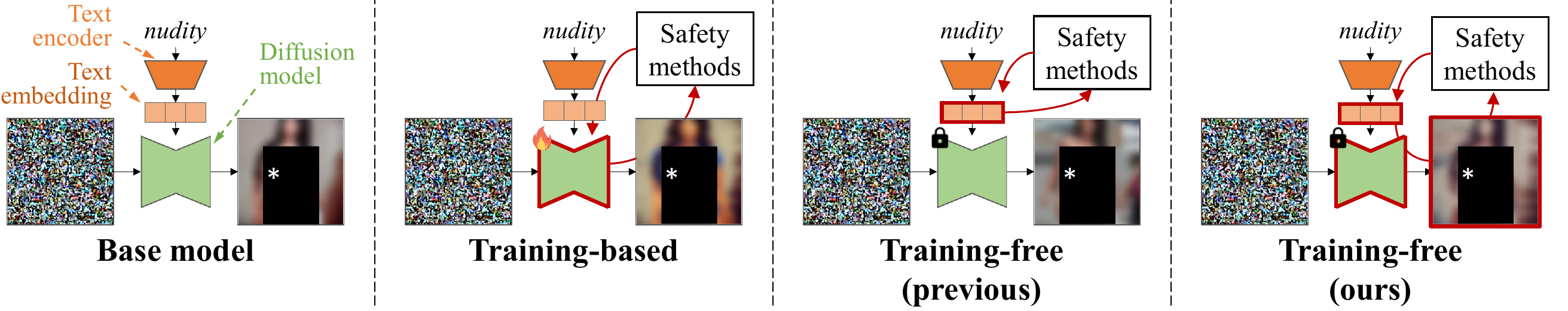}
    \caption{Overview of safe generation methods for diffusion models. Red borders indicate the components used in the safety methods. Training-based approaches fine-tune the diffusion model using additional resources and do not use samples at test time. Previous training-free methods~\cite{yoon2024safree} adjust text embeddings independently of the diffusion model. Our training-free method directly guides text embeddings using the diffusion model and its intermediate images, ensuring safer outputs without additional training. For publication purposes, the generated images are masked and blurred.}
    \vspace{-.5em}
    \label{fig:overview}
\end{figure}

\section{Related work}
\label{sec:rel}

\textbf{Training-based approaches}~~
Training-based methods for safe diffusion models involve fine-tuning the model's parameters to remove unsafe concepts~\cite{gandikota2023erasing,liu2024latent,lyu2024one,parkdirect}. For example, ESD~\cite{gandikota2023erasing} fine-tunes diffusion models by minimizing the difference between concept-conditional and unconditional outputs to erase unsafe concepts. DUO~\cite{parkdirect} performs preference optimization using the systematically generated unsafe and safe image pairs. While these training-based methods can effectively remove unwanted content, they often require additional curated training data and computational resources, and may risk degrading the model's ability to generate diverse and high-quality images.

\textbf{Training-free approaches}~~
Training-free methods aim to enable safe generation without additional training data, often by manipulating inputs or intermediate representations during inference~\cite{gandikota2024unified,gong2024reliable,schramowski2023safe,yoon2024safree}. For example, SLD~\cite{schramowski2023safe} adds the guidance using a conditional score function with unsafe text. UCE~\cite{gandikota2024unified} and RECE~\cite{gong2024reliable} adjust attention weights or text embeddings to suppress unsafe content. SAFREE~\cite{yoon2024safree} builds a subspace for unsafe token embeddings and filters embeddings that approach this subspace. However, these methods generally do not directly utilize intermediate or final diffusion states in their safety mechanisms, making their precise impact on the generated images ambiguous. They also remain vulnerable to adversarial prompts, as shown in \cref{subsec:exp_result}.

\begin{table}[tp]
    \centering
    \caption{Comparison of the safe guidance methods.}
    \adjustbox{max width=\linewidth}{%
    \begin{tabular}{llll}
        \toprule
        Method  & Guidance framework & Guidance target & Guidance module   \\
        \midrule
        SLD~\cite{schramowski2023safe}  & Classifier-Free Guidance~\cite{ho2021classifierfree} & Perturbed data $\rvx_t$ & Unsafe-cond. score network $\rvs_{\bm{\theta}}(\rvx_t,\rvc_{\text{unsafe}},t)$\\
        SG (\cref{subsec:safe_guidance})                              & Classifier Guidance~\cite{dhariwal2021diffusion} & Perturbed data $\rvx_t$ & Time-dependent classifier $g_t(\rvx_t, \rvc)$ \\
        SDG (\cref{subsec:safe_data_guidance}) & Universal Guidance~\cite{bansal2023universal} & Perturbed data $\rvx_t$ & Time-independent classifier $g(\bar{\rvx}_0(\rvx_t, \rvc))$ \\
        STG (\cref{subsec:safe_text_embedding}, ours)                      & Diffusion Adaptive Text Embedding~\cite{na2025diffusion} & Text embedding $\rvc$   & Time-independent classifier $g(\bar{\rvx}_0(\rvx_t, \rvc))$ \\
        \bottomrule
    \end{tabular}
    }
    \label{tab:compare_sg}
\end{table}

\textbf{Guidance methods in diffusion models}~~
In diffusion models, conditional generation is often implemented by adding a \textit{guidance} term to the base score function. Classifier Guidance (CG)~\cite{dhariwal2021diffusion} adds the gradient of a time-dependent classifier to inject conditional information, while Classifier-Free Guidance (CFG)~\cite{ho2021classifierfree} removes the need for an explicit classifier by using the difference between the conditional and unconditional scores. Universal Guidance (UG)~\cite{bansal2023universal} instead uses a time-independent classifier to address the need for time-dependent training. Diffusion Adaptive Text Embedding (DATE)~\cite{na2025diffusion} applies a time-independent classifier to text embeddings for better semantic alignment.

These methods can be adapted for safe generation, as summarized in \cref{tab:compare_sg}. For example, SLD~\cite{schramowski2023safe} uses a score function conditioned on unsafe text, similar to CFG. We formulate Safe Guidance (SG) analogous to CG in \cref{subsec:safe_guidance}, and Safe Data Guidance (SDG) analogous to UG in \cref{subsec:safe_data_guidance}. However, these methods rely on classifiers to estimate the safety probabilities, which are often approximated by proxy functions that can distort guidance directions and degrade generation quality, as discussed in \cref{subsec:toy}. We found that this issue is less pronounced when guidance is applied in the text embedding space, so we adopt a classifier to guide text embeddings instead of perturbed data.

\section{Preliminaries}
\label{sec:prelim}

\subsection{Diffusion models}
\label{subec:diffusion}

Diffusion models generate data by reversing a structured noise process, gradually transforming a noisy latent $\rvx_T$ into a clean data sample $\rvx_0$~\citep{ho2020denoising,song2021scorebased}. 
The forward process incrementally adds Gaussian noise to a clean sample $\rvx_0 \sim q(\rvx_0)$, forming a noisy latent $\rvx_T$ through a fixed Markov chain:
\begin{align}
    q(\rvx_{1:T}|\rvx_0) \coloneqq \textstyle \prod_{t=1}^{T} q(\rvx_{t} | \rvx_{t-1}), \text{ where } q(\rvx_{t}|\rvx_{t-1}) \coloneqq \mathcal{N}(\rvx_{t}; \sqrt{1-\beta_{t}} \rvx_{t-1}, \beta_{t} \mathbf{I}).
\end{align}
Here, $\rvx_{1:T}$ is the sequence of perturbed data, and $\beta_t$ is a pre-defined variance schedule parameter.
The reverse process progressively denoises $\rvx_T$ back to $\rvx_0$ using a parameterized Markov chain:
\begin{align}
    p_{\boldsymbol{\theta}}(\rvx_{0:T}) \coloneqq p_T(\rvx_{T}) \textstyle \prod_{t=1}^{T} p_{\boldsymbol{\theta}}(\rvx_{t-1} | \rvx_{t}), \text{ where } p_{\boldsymbol{\theta}}(\rvx_{t-1}|\rvx_{t}) \coloneqq \mathcal{N}(\rvx_{t-1}; \boldsymbol{\mu}_{\boldsymbol{\theta}}(\rvx_{t}, t), \sigma^2_{t}\rmI).
    \label{eq:backward_ddpm}
\end{align}
In this formulation, $p_T$ is typically chosen as a simple Gaussian prior, $\boldsymbol{\mu}_{\boldsymbol{\theta}}(\rvx_{t}, t)$ is the parameterized mean function, and $\sigma^2_{t}$ is a time-dependent variance parameter.

Diffusion models are trained to minimize the variational bound on the negative log-likelihood of the data~\cite{ho2020denoising}. Equivalently, the training objective can be formulated as a score matching problem~\cite{song2019generative,song2021scorebased}, where a score network $\rvs_{\boldsymbol{\theta}}(\rvx_t,t)$ is optimized to approximate the gradient of the log-density $\nabla_{\rvx_{t}} \log q_t(\rvx_t)$. 
During inference, the generation process starts by sampling a noisy latent $\rvx_T$ from the prior distribution $p_T(\rvx_T)$. Then, the noisy latent is progressively denoised by the learned reverse transitions $p_\theta(\rvx_{t-1}|\rvx_t)$, eventually producing a realistic output.

\subsection{Problem formulation: safe text-to-image diffusion models}
\label{subsec:formulation}

Text-to-image diffusion models extend the unconditional formulation by conditioning the score network on a text embedding $\rvc$ that encodes the textual prompt $y$. This enables the model to approximate the conditional score function $\nabla_{\rvx_{t}} \log q_t(\rvx_t|\rvc) $ for generating text-aligned images:
\begin{align}
    \min_{\boldsymbol{\theta}} \mathbb{E}_t \mathbb{E}_{\rvx_t \sim q_t(\rvx_t|\rvc)} \big [ || \rvs_{\boldsymbol{\theta}}(\rvx_t,\rvc,t) - \nabla_{\rvx_{t}} \log q_t(\rvx_t|\rvc) ||^2_2 \big ].
    \label{eq:text_conditional_score_matching}
\end{align}
Here, the text embedding $\rvc$ is typically obtained through a fixed, pre-trained text encoder~\cite{esser2024scaling,rombach2022high,saharia2022photorealistic}.
As we discussed earlier, the malicious usage of generative models usually comes from various prompts, requesting unsafe utilization, i.e., ranging from direct requests of inhibited generations to subtly nuanced prompts to result in such unsafe generations. 

To incorporate safety, we define a binary random variable $o \in \{ 0, 1 \}$ as a \textit{safety indicator}, where $o=1$ indicates a safe sample and $o=0$ indicates an unsafe one. Under this formulation, a safe text-to-image diffusion model aims to generate samples from the safe text-conditional distribution,
\begin{align}
    q_{\text{safe}} (\rvx_0 | \rvc) := q_0 (\rvx_0 | \rvc,o=1),
\end{align}
which defines a conditional distribution over samples that both align with the 
text condition $\rvc$ and meet the safety criterion, i.e., avoiding any improper data generation or intellectual property violations.

To sample from this distribution using a diffusion model, we need the safe text-conditional score function $\nabla_{\rvx_t} \log q_t(\rvx_t | \rvc, o=1)$. However, directly learning this score function would require a substantial number of safe text-image pairs, making this training an inefficient approach. Therefore, we propose to approximate this safe text-conditional score function without retraining the underlying diffusion model, by the text embedding $\rvc$ at inference time to guide sampling toward safer outputs.

\section{Method}
\label{sec:method}

\subsection{Safe guidance}
\label{subsec:safe_guidance}

Our goal is to achieve safe generation using a fixed pre-trained text-to-image diffusion model, as formulated in \cref{subsec:formulation}. Specifically, we aim to sample from the safe text-conditional distribution $q_t(\rvx_t|\rvc,o=1)$ without modifying the model parameters.
To enable this, we use a \textit{safety function} $g:\mathbb{R}^d \rightarrow \mathbb{R}$ that evaluates whether a clean image $\rvx_0$ is safe or not, where $d$ is the data dimension. In practice, this function can be instantiated using open-source classifiers like NudeNet~\cite{bedapudi2019nudenet} for nudity detection, or using pre-trained vision-language models like CLIP~\cite{radford2021learning}, as explained in \cref{subsec:exp_setting}.

We formulate Safe Guidance (SG) based on the CG framework~\cite{dhariwal2021diffusion}, which uses an external classifier to guide the generation process. The safe guidance for the safety indicator can be expressed as:
\begin{align}
    \nabla_{\rvx_t} \log q_t(\rvx_t|\rvc,o=1) =  \underbrace{\nabla_{\rvx_t} \log q_t(\rvx_t|\rvc)}_{\text{original text-conditional score}} + \quad \underbrace{\nabla_{\rvx_t} \log q_t(o=1|\rvx_t,\rvc)}_{\text{safe guidance}}. \label{eq:safe_guidance}
\end{align}
The first term, the original text-conditional score, can be estimated using the pre-trained score network $\rvs_{\bm{\theta}} (\rvx_t, \rvc, t)$. However, the second term, the safe guidance, is more challenging, as it requires estimating the classification probability that a given intermediate sample $\rvx_t$ is safe, conditioned on the text embedding $\rvc$. This estimation would require training on perturbed data $\rvx_t$ for all diffusion timesteps with the corresponding text conditions, resulting in significant computational overhead.

A key challenge is that the given safety function $g$ operates only on fully denoised images $\rvx_0$ and cannot be directly applied to the noisy intermediate data $\rvx_t$. In the following subsections, we first describe how the safety function $g$ can be used in the data space for safe guidance, following the Universal Guidance framework proposed in \cite{bansal2023universal}. Then, we present our proposed method, which applies the guidance in the text embedding space, effectively guiding the model toward safer outputs while preserving the core semantics of the original text condition.

\subsection{Safe data guidance (SDG)}
\label{subsec:safe_data_guidance}

We first introduce a method for applying safe guidance in the perturbed data space using the safety function $g$. This approach builds on the Universal Guidance framework~\cite{bansal2023universal}, originally developed for general diffusion-based conditional generation, and adapts it for safe generation.

To apply this method, we assume that the safety function $g(\rvx_0)$ is proportional to the safe probability distribution $q(o=1|\rvx_0)$. Under this assumption, the safe guidance term can be derived as follows:
\begin{align}
    \nabla_{\rvx_t} \log q_t(o=1|\rvx_t,\rvc) &  = \nabla_{\rvx_t} \log \mathbb{E}_{\rvx_0 \sim q(\rvx_0 | \rvx_t,\rvc)} [ q(o=1| \rvx_0) ] \label{eq:sg_1} \\
    & = \nabla_{\rvx_t}\log \mathbb{E}_{\rvx_0 \sim q(\rvx_0 | \rvx_t,\rvc)} [ g(\rvx_0) ]  \label{eq:sg_2} \\
    & \approx \nabla_{\rvx_t} \log g ( \mathbb{E}_{\rvx_0 \sim q(\rvx_0 | \rvx_t,\rvc)} [\rvx_0]) \label{eq:sg_3} \\
    & \approx \nabla_{\rvx_t} \log g \Big (\frac{1}{\sqrt{\bar{\alpha}_t}} \big (\rvx_t + (1-\bar{\alpha}_t) \rvs_{\bm{\theta}} (\rvx_t, \rvc, t) \big ) \Big ),  \label{eq:sg_4}
\end{align}
where $\bar{\alpha}_t \coloneqq \prod_{\tau=1}^t (1-\beta_\tau)$ is a constant determined by the variance schedule of the forward process. The detailed derivation is provided in \cref{app_subsec:derivation} and involves applying a first-order Taylor approximation and Tweedie's formula~\cite{efron2011tweedie}, similar to the approach used in the Universal Guidance.

From this derivation, we introduce the resulting safe guidance method as \textit{Safe Data Guidance (SDG)}:
\begin{align}
    \rvs_{\text{SDG}} (\rvx_t, \rvc, t) := \rvs_{\bm{\theta}} (\rvx_t, \rvc, t) + \nabla_{\rvx_t} \log g \Big (\frac{1}{\sqrt{\bar{\alpha}_t}} \big (\rvx_t + (1-\bar{\alpha}_t) \rvs_{\bm{\theta}} (\rvx_t, \rvc, t) \big ) \Big ). \label{eq:sdg}
\end{align}
SDG provides a straightforward approach to safe generation by encouraging the sample $\rvx_t$ to move in a direction that increases the safety score evaluated by $g$ at the expected denoised output. However, this method relies on the assumption that $g$ is exactly proportional to the safe probability $q(o=1|\rvx_0)$, rather than preserving the same order. Even if $g$ can preserve the relative order of safe and unsafe samples, the difference in the shape of the function can lead to a generated distribution that deviates from the original text-conditional distribution $q(\rvx_0|\rvc)$. We further analyze this issue in \cref{subsec:toy}.

\subsection{Safe text embedding guidance (STG)}
\label{subsec:safe_text_embedding}

As an alternative to direct guidance in the data space, we propose a method that applies guidance to the text embedding using the safety function $g$. In text-conditional generation, unsafe outputs often arise from malicious text prompts~\cite{yang2024guardti}. To address this, we aim to adjust the text embedding $\rvc$ toward a safer representation, encouraging the generation of safer images.

We define a \textit{time-dependent safety function} $g_t(\rvx_t, \rvc)$ that estimates the expected safety score of the final denoised output $\rvx_0$, given the current perturbed sample $\rvx_t$ and the text embedding $\rvc$:
\begin{align}
    g_t(\rvx_t, \rvc) := \mathbb{E}_{\rvx_0 \sim q(\rvx_0 | \rvx_t, \rvc)} [ g(\rvx_0) ]. \label{eq:tdsf}
\end{align}
This function captures the expected safety of the final image output, conditioned on the current noisy sample and the text representation.
To guide the text embedding $\rvc$ toward safer outputs, we apply gradient ascent on this time-dependent safety function $g_t(\rvx_t, \rvc)$. Specifically, we update $\rvc$ as: 
\begin{align}
    \rvc \leftarrow \rvc + \rho \nabla_{\rvc} g_t(\rvx_t, \rvc),
\end{align}
where $\rho$ is the scale hyperparameter that controls the strength of the safety adjustment. The updated text embedding is then used in the score network to perform diffusion sampling.

To make this guidance method tractable, we approximate the time-dependent safety function $g_t(\rvx_t, \rvc)$ using the same logic of the derivation in \cref{eq:sg_2,eq:sg_3,eq:sg_4}:
\begin{align}
    g_t(\rvx_t, \rvc) \approx g (\mathbb{E}_{\rvx_0 \sim q(\rvx_0 | \rvx_t, \rvc)} [ \rvx_0 ] ) \approx  g \Big (\frac{1}{\sqrt{\bar{\alpha}_t}} \big (\rvx_t + (1-\bar{\alpha}_t) \rvs_{\bm{\theta}} (\rvx_t, \rvc, t) \big ) \Big ). \label{eq:tdsf_approx}
\end{align}

Based on this tractable form, we propose \textit{Safe Text Embedding Guidance (STG)}, which uses the score network with the updated text embedding:
\begin{align}
    \rvs_{\text{STG}} (\rvx_t, \rvc, t) := \rvs_{\bm{\theta}} \Big (\rvx_t, \rvc + \rho \nabla_{\rvc} g \Big (\frac{1}{\sqrt{\bar{\alpha}_t}} \big (\rvx_t + (1-\bar{\alpha}_t) \rvs_{\bm{\theta}} (\rvx_t, \rvc, t) \big ) \Big ), t \Big ). \label{eq:stg}
\end{align}

\paragraph{Analysis of STG on data space}

Since STG applies guidance to the text embedding $\rvc$, it implicitly influences the perturbed data $\rvx_t$ as well. In \cref{thm:main}, we analyze the impact of STG, which updates the text embedding, from the perspective of the perturbed data $\rvx_t$.

\begin{restatable}{theorem}{thmmain}
\label{thm:main}
    Let $q_t(\rvx_t|\rvc)$ be the text-conditional distribution at diffusion timestep $t$, and $g_t(\rvx_t, \rvc)$ be a time-dependent safety function at $t$.
    If the text embedding $\rvc$ is updated using STG with the step size $\rho$, then the resulting score function can be expressed as:
    \begin{align}
        \nabla_{\rvx_t} \log q_t(\rvx_t | \rvc & + \rho \nabla_\rvc g_t(\rvx_t, \rvc)) \nonumber \\
        & = \underbrace{\nabla_{\rvx_t} \log q_t (\rvx_t | \rvc)}_{\text{original text-conditional score}} + \quad  \underbrace{\nabla_{\rvx_t} \{ \rho  \nabla_\rvc g_t(\rvx_t, \rvc)^T \nabla_\rvc \log q_t(\rvx_t|\rvc)  \} }_{\text{safe guidance}} + O(\rho^2).
    \end{align}
\end{restatable}
The proof is provided in \cref{app_subsec:thm_main}. The adjusted score function $\nabla_{\rvx_t} \log q_t(\rvx_t | \rvc + \rho \nabla_\rvc g_t(\rvx_t, \rvc))$ is decomposed into the original text-conditional score and the safe guidance term, analogous to SG in \cref{eq:safe_guidance}. Therefore, STG can be interpreted as a form of SG in which the safe conditional probability $q_t(o=1|\rvx_t,\rvc)$ is defined by the alignment between the gradient of the safety function $g_t$ and the text-conditional likelihood $q_t(\rvx_t|\rvc)$:
\begin{align}
    q_t^{\text{STG}}(o=1|\rvx_t,\rvc) \appropto \exp \Big ( \rho  \nabla_\rvc g_t(\rvx_t, \rvc)^T \nabla_\rvc \log q_t(\rvx_t|\rvc)  \Big ). \label{eq:safe_prob_stg}
\end{align}
This implies that STG sets the safe probability for intermediate samples by aligning the underlying model likelihood with the desired safety objective. As a result, STG simultaneously preserves the original distribution of the base model and guides the generation toward safer outputs. This approach maintains the core semantics of the generated content while reducing the likelihood of unsafe results.

\subsection{Comparison between SDG and STG}
\label{subsec:toy}

\begin{figure}[t]
    \centering
    \begin{subfigure}[b]{\linewidth}
        \centering
        \includegraphics[width=0.97\linewidth]{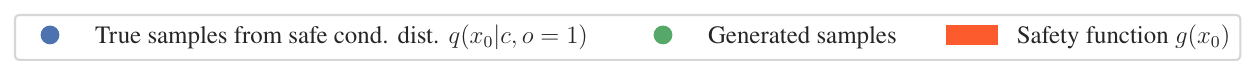}
    \end{subfigure}
    
    \centering
    \begin{subfigure}[b]{0.19\linewidth}
        \centering
        \includegraphics[width=\linewidth]{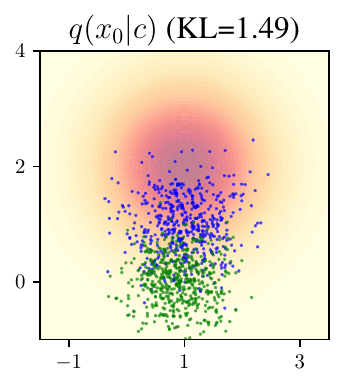}
        \caption{No safe guidance}
        \label{subfig:toy_0}
    \end{subfigure}
    \hfill 
    \begin{subfigure}[b]{0.375\linewidth}
        \centering
        \includegraphics[width=\linewidth]{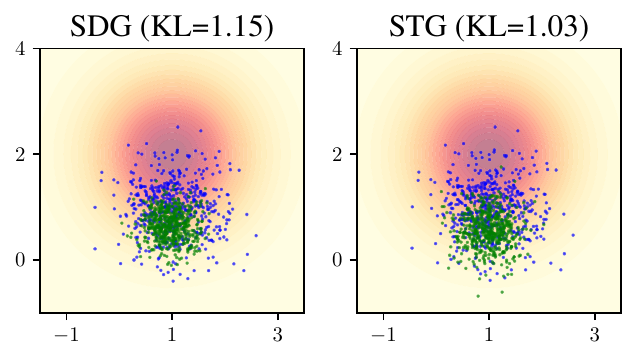}
        \caption{Safe guidance using $g^*$}
        \label{subfig:toy_1}
    \end{subfigure}
    \hfill 
    \begin{subfigure}[b]{0.375\linewidth}
        \centering
        \includegraphics[width=\linewidth]{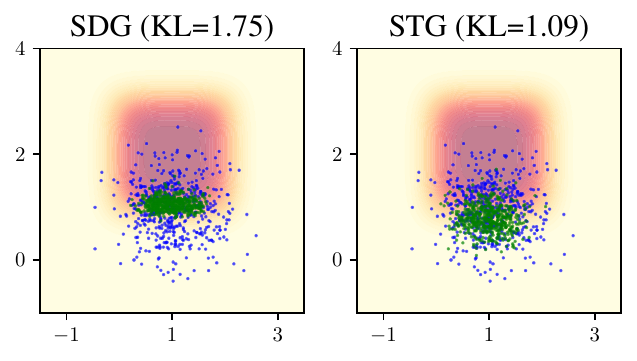}
        \caption{Safe guidance using $\tilde{g}$}
        \label{subfig:toy_2}
    \end{subfigure}
    \caption{Generated samples from the 2D toy example with condition $\rvc=(1,0)$ using SDG and STG with different safety functions. $g^*$ is the ideal safety function, proportional to the true safe distribution $p(o=1|\rvx_0)$, while the approximated safety function $\tilde{g}$ preserves relative order but differs in shape. The blue dots represent samples from the true safe conditional distribution $q(\rvx_0|\rvc,o=1)$, and the green dots indicate instances generated using each guidance method. The background heatmap shows the contours of the respective safety functions. The value in parentheses in each figure title indicates the KL divergence between the true safe conditional distribution and the generated samples.}
    \label{fig:toy}
\end{figure}

\textbf{Toy experiment setup}~~
To compare safe guidance variants, we use a 2D toy example where the true target distribution and guidance terms can be computed tractably. The conditional distribution is defined as a 2D Gaussian with the condition $\rvc\in \mathbb{R}^2$ as its mean: $q(\rvx_0|\rvc) = \mathcal{N}(\rvx_0 ; \rvc, \rmI)$. The safe distribution is defined as $q(o=1|\rvx_0) \propto \exp (-\frac{1}{2} ||\rvx_0 - \bm{\mu}||^2 )$ where $\bm{\mu}=(1, 2)$, as shown in \cref{subfig:toy_0}. Under this setup, the true safe conditional distribution is $q(\rvx_0 | \rvc, o=1) = \mathcal{N} (\rvx_0 ; \frac{1}{2} (\rvc + \bm{\mu}), \frac{1}{2}\rmI)$.

We consider two safety functions: 1) \textit{ideal} safety function $g^*(\rvx_0) = \exp( -\frac{1}{2} ||\rvx_0 -\bm{\mu}||^2)$, which is proportional to the true safe distribution, satisfying the assumptions of SDG; and 2) \textit{approximated} safety function $\tilde{g}(\rvx_0) = \exp( -\frac{1}{2} ||\rvx_0 -\bm{\mu}||^4) $, which preserves the relative safety ordering but deviates in its shape. The contours of each safety function are shown in \cref{subfig:toy_1,subfig:toy_2}.

\textbf{Results}~~
We present the generated samples using different safety functions for each guidance method in \cref{fig:toy}. Without guidance (\cref{subfig:toy_0}), samples are drawn from the original $\rvc$-conditional distribution, centered around $\rvc=(1, 0)$.
As safe guidance is applied (\cref{subfig:toy_1,subfig:toy_2}), instances shift toward the safe region as expected. With $g^*$, SDG effectively guides the samples to the correct safe region because it fully satisfies the assumption required for accurate guidance. In contrast, with $\tilde{g}$, SDG produces more biased samples, as indicated by the larger KL divergence. This is because SDG directly relies on the provided safety function without correcting for potential mismatches with the true safe distribution. As a result, SDG may push samples toward regions that satisfy $\tilde{g}$  but deviate from the true safe $\rvc$-conditional distribution.

In contrast to the biased generation in the case of SDG using $\tilde{g}$, STG shows more robust performance with both safety functions, as it accounts for both the underlying model likelihood and the safe direction. This generates samples that better preserve the underlying model distribution, reducing mode collapse and improving overall sample quality.

\section{Experiments}
\label{sec:exp}
\subsection{Experimental settings}
\label{subsec:exp_setting}

\begin{figure}[t]
    \centering
    \begin{subfigure}[b]{\linewidth}
        \centering
        \includegraphics[width=0.9\linewidth]{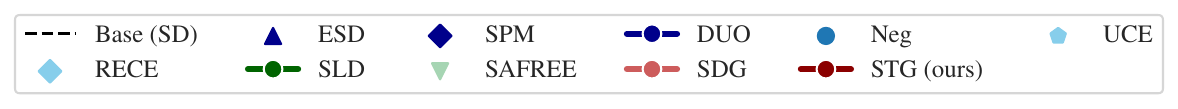}
    \end{subfigure}
    \begin{subfigure}[b]{0.32\linewidth}
        \centering
        \includegraphics[width=\linewidth]{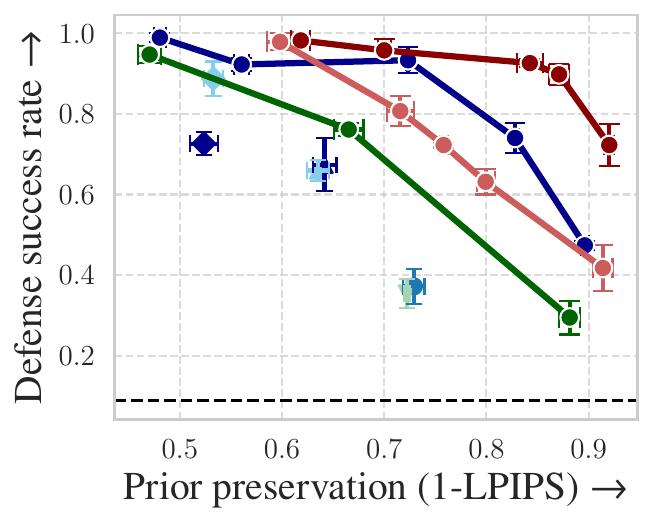}
        \caption{Nudity: Ring-A-Bell}
        \label{subfig:black_ringabell}
    \end{subfigure}
    \hfill 
    \begin{subfigure}[b]{0.32\linewidth}
        \centering
        \includegraphics[width=\linewidth]{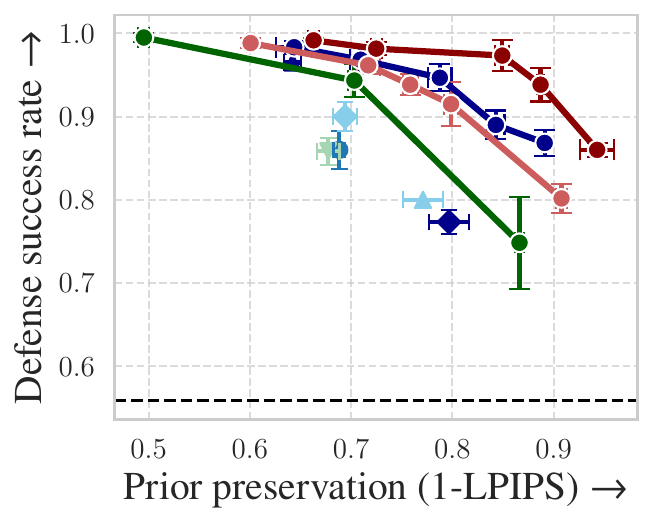}
        \caption{Nudity: SneakyPrompt}
        \label{subfig:black_sneakyprompt}
    \end{subfigure}
    \hfill 
    \begin{subfigure}[b]{0.32\linewidth}
        \centering
        \includegraphics[width=\linewidth]{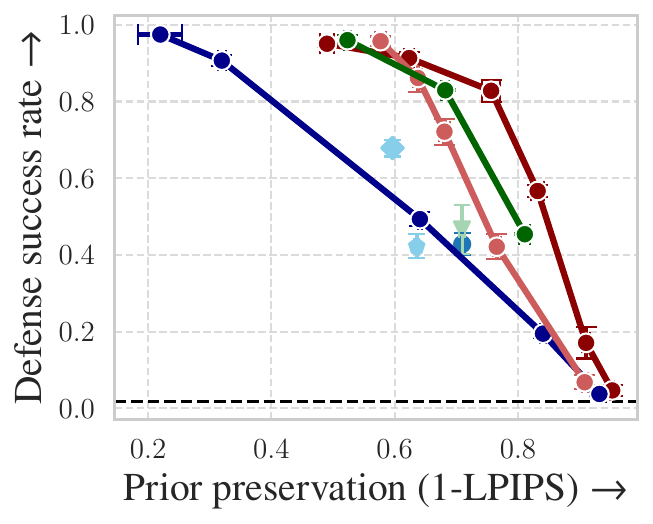}
        \caption{Violence: Ring-A-Bell}
        \label{subfig:black_vio}
    \end{subfigure}
    \caption{Trade-off between defense success rate and prior preservation on nudity and violence. Each experiment is repeated three times with different random seeds, and the mean values are shown as points while the standard deviations are indicated by error bars.}
    \label{fig:nudity}
\end{figure}

\textbf{Setup}~~
Following previous work~\cite{parkdirect,yoon2024safree}, we mainly use the publicly available Stable Diffusion v1.4~\cite{rombach2022high} as the backbone architecture. Sampling is performed with a DDIM sampler~\cite{song2021denoising} with 50 steps and a classifier-free guidance scale of 7.5. To further evaluate the generalization ability of our approach, we additionally conduct experiments with diverse backbone models, including FLUX~\cite{flux2024}, SDXL~\cite{podell2024sdxl}, SD3~\cite{esser2024scaling}, and PixArt-$\alpha$~\cite{chen2024pixartalpha}, as well as with different samplers, such as DDPM~\cite{ho2020denoising}.

We evaluate our method on \textit{nudity} and \textit{violence} using both black-box and white-box red-teaming protocols, following \cite{parkdirect}. For black-box attacks, we use Ring-A-Bell~\cite{tsai2023ring} (95 nudity and 250 violence prompts) and SneakyPrompt~\cite{yang2024sneakyprompt} (200 nudity prompts). 
For white-box attacks, we adopt Concept Inversion~\cite{pham2023circumventing}, where a special token <c> is learned via textual inversion to bypass safety mechanisms. This setup tests whether training-free methods can provide an additional defense when combined with the training-based approach.
Following \cite{gandikota2023erasing,yoon2024safree}, we also evaluate the models on an \textit{artist-style removal} task using two sets of 100 prompts, each consisting of 20 prompts for five different artist styles that Stable Diffusion is known to mimic. Further details are in \cref{app_subsec:exp_setup}.

\textbf{Implementation details for STG}~~
We define the safety function $g$ for STG as follows. For \textit{nudity}, $g$ is set to the negative sum of the confidence scores of bounding boxes labeled as nudity by the NudeNet detector~\cite{bedapudi2019nudenet}. For \textit{violence}, $g$ is defined as the negative CLIP score \cite{hessel2021clipscore} between the generated image and a pre-defined violence-related text. For \textit{artist-style removal}, $g$ is computed as the difference between the CLIP score of the image with the text `art' and that with the target artist's name.

To control the strength of the safety guidance, we adjust the update scale hyperparameter $\rho$. Additionally, we introduce two hyperparameters, the update threshold $\tau$ and the update step ratio $\gamma$, to reduce computational cost. The threshold $\tau$ determines whether guidance is applied based on the estimated safety value at each diffusion timestep, and the ratio $\gamma$ specifies how frequently the safety update is performed during sampling, providing a controllable trade-off between efficiency and safety performance. The detailed hyperparameter settings are provided in \cref{app_subsec:exp_detail_stg}.

\textbf{Baseline}~~
We compare our method with both training-free and training-based safety approaches. For training-free baselines, we include UCE~\cite{gandikota2024unified}, RECE~\cite{gong2024reliable}, SLD~\cite{schramowski2023safe}, and SAFREE~\cite{yoon2024safree}. In addition, we also evaluate Negative Prompt, which replaces the null prompt with an unsafe prompt in the classifier-free guidance framework, as well as SDG proposed in \cref{subsec:safe_data_guidance}. For training-based methods, we evaluate against ESD~\cite{gandikota2023erasing}, SPM~\cite{lyu2024one}, and DUO~\cite{parkdirect}. Detailed descriptions of these baselines and their implementations can be found in \cref{app_subsec:baseline}. 

\textbf{Evaluation}~~
We measure the performance of our method using the following key metrics. (1) \textit{Defense success rate} (DSR) measures the effectiveness of the safety mechanism in suppressing sensitive content. For nudity, DSR is calculated using the NudeNet Detector~\cite{bedapudi2019nudenet}. An image is considered safe if the nudity labels are not detected. For violence, we use GPT-4o~\cite{hurst2024gpt} to assess whether the generated content is potentially offensive or distressing, based on a prompt from the previous work~\cite{parkdirect}. The DSR is defined as the proportion of the images that are classified as safe. To validate the robustness of our evaluation metrics, we further report alternative results (Falconsai NSFW image classifier~\cite{falconsai2024nsfw} for nudity, Q16 classifier~\cite{schramowski2022can} for violence) in \cref{app_subsec:falconsai}. (2) \textit{Prior Preservation} (PP) measures the level of maintenance of the original generative capabilities by evaluating the perceptual similarity between outputs from the original model and those generated with safety methods. PP is computed as the average value of $1-\text{LPIPS}$, where LPIPS~\cite{zhang2018unreasonable} measures the perceptual distance between paired images. (3) \textit{General generation quality} is assessed using zero-shot FID \cite{heusel2017gans} and CLIP score on 3,000 images generated from randomly sampled captions in the COCO validation set, capturing overall image fidelity and text-image alignment. When FID is computed using 1,000 generated images, we denote it as FID-1K. Further details on evaluation protocols can be found in \cref{app_subsec:eval_details}. 

\begin{figure}[tp]
    \centering
    \includegraphics[width=0.92\linewidth]{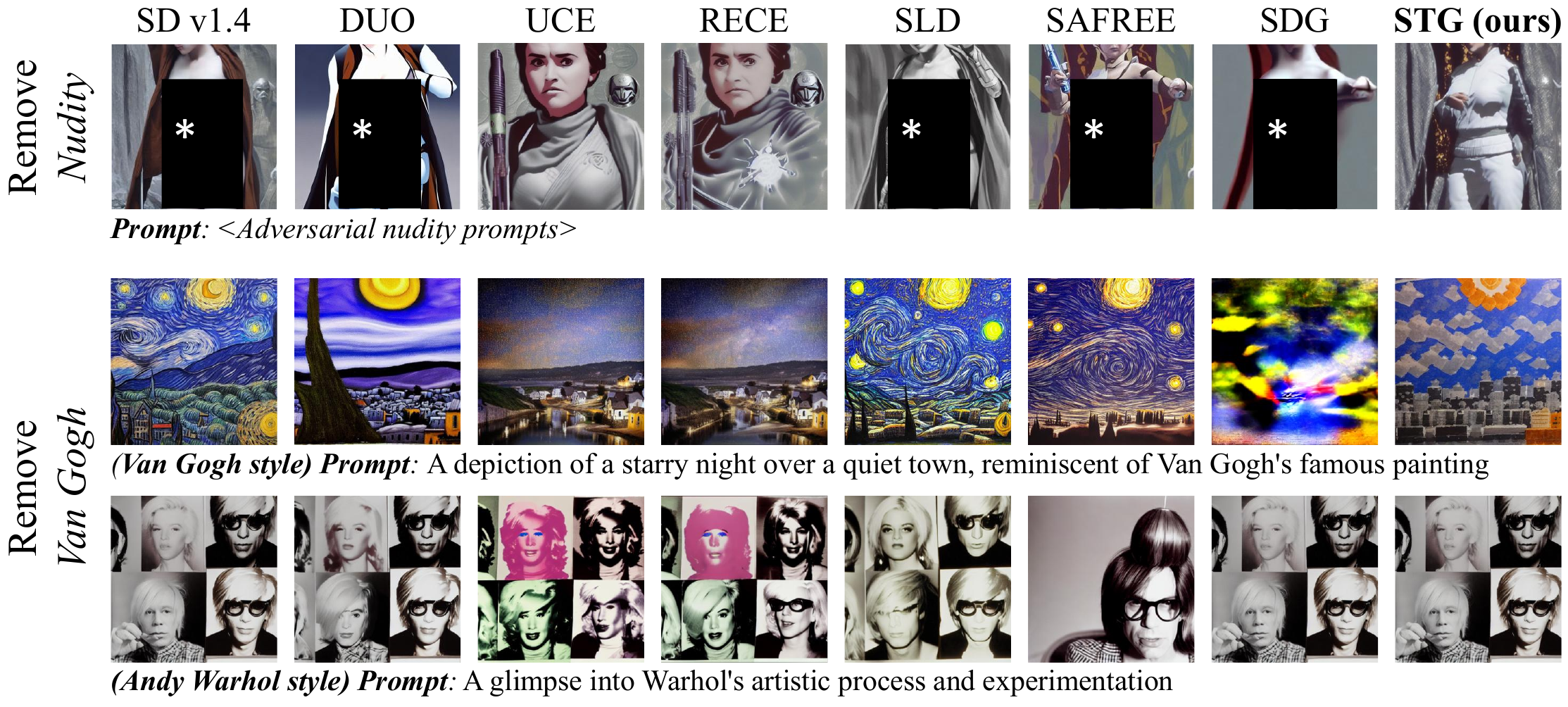}
    \caption{Generated images from STG and other safe generation baselines for nudity and artist-style removal scenarios. For publication purposes, the generated images are masked.}
    \label{fig:example}
\end{figure}

\subsection{Experimental results}
\label{subsec:exp_result}

\begin{figure}[t]
    \begin{minipage}[t]{0.43\textwidth}
        \centering
        \captionof{table}{Results for generation quality on the COCO dataset across various safe generation methods applied for nudity removal.}
        \adjustbox{max width=\linewidth}{%
        \begin{tabular}{clcc}
            \toprule
            & Method  & FID $\downarrow$ & CLIP $\uparrow$  \\
            \midrule
            & Base (SD v1.4) & 23.22 & 31.96  \\
            \midrule
            \multirow{3}{*}{\shortstack{Training\\-based}} & ESD~\cite{gandikota2023erasing}&22.85&30.02 \\
            & SPM~\cite{lyu2024one} &23.53&31.67\\
            & DUO~\cite{parkdirect} &23.27&\bf{31.90}\\
            \midrule
            \multirow{7}{*}{\shortstack{Training\\-free}}  &  Negative Prompt &24.83&31.01\\
            & UCE~\cite{gandikota2024unified}  &23.20 & 31.71\\
            & RECE~\cite{gong2024reliable}  & 23.15 & 31.07 \\
            & SLD~\cite{schramowski2023safe}  &24.32&31.29\\
            & SAFREE~\cite{yoon2024safree}  &28.39&30.27\\
            \cmidrule{2-4}
            & SDG  &26.90&29.97\\
            & \textbf{STG (ours)}  &\bf{22.00} &31.14\\
            \bottomrule
        \end{tabular}
        }
        \label{tab:coco}
    \end{minipage}
    \hfill
    \begin{minipage}[t]{0.55\textwidth}
        \vspace{-1em}
        \begin{figure}[H]
            \centering
            \begin{subfigure}[b]{0.95\linewidth}
                \centering
                \includegraphics[width=\linewidth]{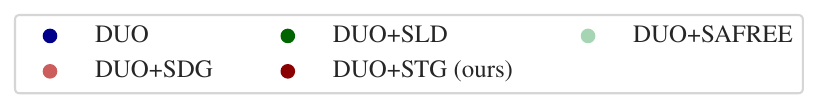}
            \end{subfigure}
            \begin{subfigure}[b]{0.48\linewidth}
                \centering
                \includegraphics[width=\linewidth]{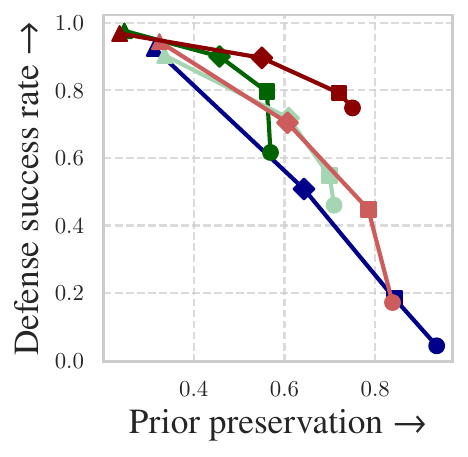}
                \caption{Ring-A-Bell}
                \label{fig:vio_duo}
            \end{subfigure}
            \hfill 
            \begin{subfigure}[b]{0.48\linewidth}
                \centering
                \includegraphics[width=\linewidth]{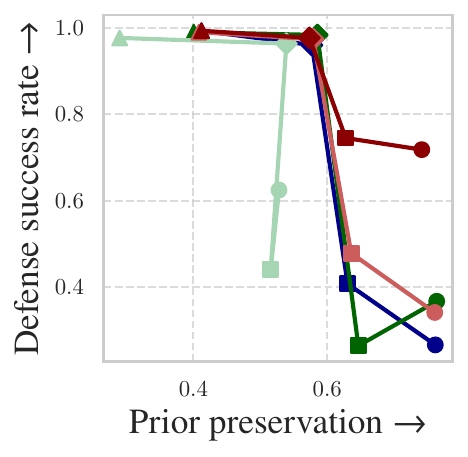}
                \caption{Concept Inversion}
                \label{fig:white_vio_duo}
            \end{subfigure}
            \caption{Results for DUO with training-free methods on violence, with fixed parameters for those methods.}
            \label{fig:with_duo}
        \end{figure}
    \end{minipage}
\end{figure}

\begin{table}[tp]
    \centering
    \caption{Results across backbone models (FLUX, SDXL, SD3) and the fast generation model LCM, demonstrating the generalization ability of STG. DSR and PP are reported on Ring-A-Bell (violence), while COCO FID-1K and CLIP score measure general image quality.}
    \begin{tabular}{lcccccccc}
        \toprule
        & \multicolumn{4}{c}{FLUX~\cite{flux2024}} & \multicolumn{4}{c}{SD3~\cite{esser2024scaling}} \\
        \cmidrule(lr){2-5}\cmidrule(lr){6-9}
        Method & DSR$\uparrow$ & PP$\uparrow$& FID-1K$\downarrow$ & CLIP$\uparrow$ & DSR$\uparrow$ & PP$\uparrow$& FID-1K$\downarrow$ & CLIP$\uparrow$  \\
        \midrule
        Base    & 0.11 & - & 56.58 & 32.67   & 0.12 & - & 53.70 & 33.44      \\
        \midrule
        \bf{STG (ours)}\\
        ~$\tau=0.20$ & 0.28 & 0.85 & 56.59 & 32.67  & 0.42 & 0.76 & 53.89 & 33.42   \\
        ~$\tau=0.18$ & 0.53 & 0.70 & 56.52 & 32.60 & 0.54 & 0.67 & 53.65 & 33.33  \\
        ~$\tau=0.16$ & 0.70 & 0.60 & 57.77 & 32.00 & 0.68 & 0.57 & 54.91 & 32.93   \\
        \midrule
        & \multicolumn{4}{c}{SDXL~\cite{podell2024sdxl}} & \multicolumn{4}{c}{LCM~\cite{luo2023latent}} \\
        \cmidrule(lr){2-5}\cmidrule(lr){6-9}
        Method & DSR$\uparrow$ & PP$\uparrow$& FID-1K$\downarrow$ & CLIP$\uparrow$ & DSR$\uparrow$ & PP$\uparrow$& FID-1K$\downarrow$ & CLIP$\uparrow$  \\
        \midrule
        Base    & 0.04 & - & 48.97 & 33.80  & 0.02 & - &  60.87 & 30.19   \\
        \midrule
        \bf{STG (ours)}\\
        ~$\tau=0.20$ & 0.25 & 0.88 & 49.24 & 33.78 & 0.23 & 0.66 & 60.96 & 30.13  \\
        ~$\tau=0.18$ & 0.50 & 0.80 & 49.11 & 33.66  & 0.52 & 0.59 & 61.05 & 30.06\\
        ~$\tau=0.16$ & 0.77 & 0.80 & 49.44 & 33.02 & 0.80 & 0.48 & 62.32 & 29.23 \\
        \bottomrule
    \end{tabular}
    \label{tab:gen}
\end{table}

\paragraph{Nudity and violence}

\cref{fig:nudity} presents the quantitative results for the nudity and violence scenarios, illustrating the trade-off between DSR and PP across various black-box attacks. To verify the significance of these results, we repeat each experiment three times with different random seeds, which correspond to variations in the initial noise during sampling.
The mean values and standard deviations of DSR and PP are shown as points and error bars, respectively.
Examples of generated images are shown in \cref{fig:example}. STG consistently occupies the upper-right corner of the trade-off curve, indicating its superior ability to effectively filter unsafe concepts while preserving original information that is unrelated to safety.
In the nudity cases (\cref{subfig:black_ringabell,subfig:black_sneakyprompt}), the training-based DUO achieved the best performance among the baselines, but its effectiveness is reduced on violence (\cref{subfig:black_vio}). As noted in the DUO work, this is likely due to the diverse categories of violence, which are harder to capture all potentially unsafe concepts through training. In contrast, SDG and STG demonstrate strong performance in both scenarios, leveraging test-time CLIP score guidance based on violence-related text to better handle the broader range of violence concepts.

\cref{tab:coco} shows the quantitative results on the COCO dataset, which generally contains safe prompts. STG demonstrates superior generation quality, achieving the lowest FID, even outperforming the base model, though with a slight drop in CLIP score due to the text embedding modification. This suggests that STG preserves the overall diversity and realism of the original model, benefiting from the likelihood-preserving term in its guidance, as discussed in \cref{subsec:safe_text_embedding}. Therefore, STG effectively filters unsafe content while minimizing unintended degradation of the model's generative capacity. 

Since training-free methods can be combined with training-based approaches, we conduct experiments applying various training-free methods to DUO~\cite{parkdirect}, the best-performing training-based safe generation methods. \cref{fig:with_duo} presents the results for the violence under both black-box and white-box red teaming, where each training-free method is applied to DUO models with different parameter settings, while keeping the parameters of the training-free methods fixed. In these settings, our method outperforms other training-free baselines, highlighting its adaptability.

To evaluate the generalization ability of STG, we further test it on recent diffusion backbones, FLUX~\cite{flux2024}, SDXL~\cite{podell2024sdxl}, and SD3~\cite{esser2024scaling}, using their default configurations on the Ring-A-Bell (violence) benchmark. We also include PixArt-$\alpha$~\cite{chen2024pixartalpha} with DPM-Solver~\cite{lu2022dpm}, whose results are provided in \cref{tab:pixart} of \cref{app_subsec:pixart}. As summarized in \cref{tab:gen}, the base models still produce harmful outputs for violence-related prompts, while STG consistently improves DSR while maintaining comparable overall generation quality. These results demonstrate strong generalization across diverse backbones, with a controllable safety-quality trade-off via the scale hyperparameter $\rho$. Moreover, STG integrates seamlessly with fast generation models such as LCM~\cite{luo2023latent}, since STG only requires access to the mean predicted images at intermediate timesteps, which are readily available in most diffusion frameworks. Additional experiments with different samplers, including DDPM~\cite{ho2020denoising}, demonstrate that STG remains robust across sampling strategies, as shown in \cref{fig:ddpm} of \cref{app_subsec:ddpm}.

\paragraph{Artist-style removal}

\begin{table}[tp]
    \centering
    \caption{Quantitative comparison of artist-style removal on famous (left) and modern (right) artists.}
    \resizebox{\textwidth}{!}{%
    \begin{tabular}{lcccccccc}
        \toprule
        & \multicolumn{4}{c}{Remove ``Van Gogh''} & \multicolumn{4}{c}{Remove ``Kelly McKernan''} \\
        \cmidrule(lr){2-5}\cmidrule(lr){6-9}
        Method & $\text{LPIPS}_e$$\uparrow$ & $\text{LPIPS}_u$$\downarrow$ & $\text{ACC}_e$$\downarrow$ & $\text{ACC}_u$$\uparrow$
               & $\text{LPIPS}_e$$\uparrow$ & $\text{LPIPS}_u$$\downarrow$ & $\text{ACC}_e$$\downarrow$ & $\text{ACC}_u$$\uparrow$ \\
        \midrule
        Base (SD v1.4)                         & --   & --   & 1.00 & 0.89 & --   & --   & 0.90 & 0.71 \\
        \midrule
        DUO~\cite{parkdirect}                  & 0.38 & 0.17 & 0.60 & 0.90 & 0.42 & 0.26 & 0.55 & 0.70 \\
        \midrule
        UCE~\cite{gandikota2024unified}        & 0.36 & 0.18 & 0.45 & \bf{0.95} & 0.40 & 0.17 & 0.35 & 0.73 \\
        RECE~\cite{gong2024reliable}           & 0.36 & 0.19 & 0.60 & 0.93 & 0.42 & 0.17  &  0.25  &  0.71 \\
        SLD~\cite{schramowski2023safe}  & 0.28 & 0.12 & 0.60 & 0.81 & 0.22 & 0.18 & 0.50 & \bf{0.74} \\
        SAFREE~\cite{yoon2024safree}           & 0.39 & 0.25 & 0.45 & 0.75 & 0.47 & 0.46 & 0.25 & 0.71 \\
        \midrule
        SDG                                    & 0.43 & 0.09 & 0.30 & 0.83 & 0.42 & 0.11 &  0.30  &  0.68  \\
        \textbf{STG (ours)}                    & \bf{0.46} & \bf{0.08} & \bf{0.30} & 0.85 & \bf{0.58} & \bf{0.10} & \bf{0.10}  &  0.65  \\
        \bottomrule
    \end{tabular}%
    }
    \label{tab:main_artist}
\end{table}

To evaluate artist-style removal, we follow the protocol from \cite{yoon2024safree} using two metrics: LPIPS and ACC. LPIPS measures the average perceptual distance between images from the base model and those produced by the safe method. ACC is the average accuracy with which GPT-4o identifies the specified artist style in the prompt. The subscripts ``\textit{e}'' and ``\textit{u}'' on each metric denote the evaluated prompt sets, which are ``\textit{erased}'' (target style removed) and ``\textit{unerased}'' (other styles), respectively. High $\text{LPIPS}_e$ and low $\text{ACC}_e$ indicate effective target style removal. Low $\text{LPIPS}_u$ and high $\text{ACC}_u$ show preservation of non-targeted styles, maintaining the original model's capabilities.

\cref{tab:main_artist} reports the quantitative results for each safe method. Examples of images generated from erased and unerased prompts are provided in \cref{fig:example}. Both SDG and STG effectively remove the target style while retaining other styles, compared to all baselines. This success stems from measuring a safety value on intermediate latents. Consequently, our approach can consistently remove the target artist's style at test time without degrading the model's overall generative performance.

We also demonstrate in \cref{app_subsec:bias} that STG can be flexibly extended to bias mitigation tasks.

\subsection{Analysis of STG}
\label{subsec:abl}

\begin{figure}[t]
    \centering
    \begin{minipage}[t]{0.5\textwidth}
    \begin{figure}[H]
        \includegraphics[width=\linewidth]{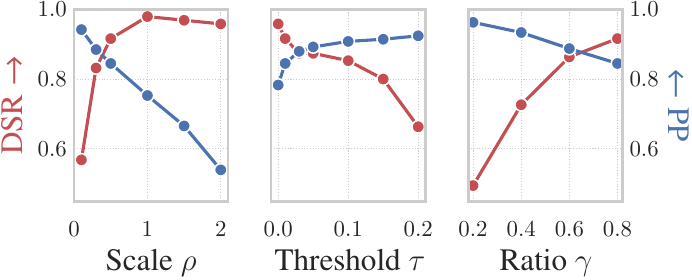}
        \caption{Sensitivity analysis of STG on Ring-A-Bell (nudity) with respect to the scale hyperparameter~$\rho$, update threshold~$\tau$, and update step ratio~$\gamma$.}
        \label{fig:sens}
    \end{figure}
    \end{minipage}
    \hfill
    \begin{minipage}[t]{0.46\textwidth}
        \captionof{table}{Sampling time (s/batch, batch size=4) and GPU memory usage (GB) with FP16.}
        \setlength{\tabcolsep}{5pt}
        \adjustbox{max width=\linewidth}{%
            \begin{tabular}{lccccc}
                \toprule
                Method & FP16 & Time & Memory & DSR$\uparrow$ & PP$\uparrow$    \\
                \midrule
                Base         & \xmark   & 15.8 & 8.23 & 0.08 & -      \\
                \midrule
                SLD~\cite{schramowski2023safe} & \xmark & 22.7 & 13.5  & 0.76 & 0.65 \\ 
                SAFREE~\cite{yoon2024safree}  & \xmark & 23.2 & 13.6 & 0.36 & 0.73 \\
                \midrule
                \bf{STG (ours)}\\
                ~$\rho=2.0$, $\tau=0.15$  & \xmark &  23.9 & 45.4 & 0.79 & \bf{0.90} \\
                  & \cmark &  14.0 & 22.8 & 0.79 & 0.89 \\
                ~$\rho=2.0$, $\tau=0.40$  & \xmark & 30.4   & 45.4   & 0.88 & 0.84    \\
                 & \cmark & 20.7   & 22.8   & \bf{0.92} & 0.84    \\
                ~$\rho=0.5$, $\tau=0.80$  & \xmark &  50.8  & 45.4  & \bf{0.92} & 0.84  \\
                & \cmark &  35.0  & 22.8  & 0.91 & 0.84  \\
                \bottomrule
            \end{tabular}
        \label{tab:time}
        }
    \end{minipage}
    \vspace{-.5em}
\end{figure}

We conduct sensitivity analyses of STG hyperparameters on the Ring-A-Bell (nudity), with the results shown in \cref{fig:sens}. The scale hyperparameter $\rho$ controls the strength of the guidance applied to the text embeddings. As $\rho$ increases, the guidance effect becomes stronger, resulting in higher DSR but greater deviation from the original image. This enables adjustment of the modification strength based on the desired safety level. The update threshold $\tau$ sets the minimum safety value for applying an update at each sampling step. Lower $\tau$ values increase the update frequency, leading to higher DSR but lower PP. A similar trend is observed with the update step ratio $\gamma$, which controls the proportion of steps where updates are applied. These hyperparameters impact the overall generation time, as shown in \cref{tab:time}, which reports sampling times for the training-free methods. The results show that our method achieves superior performance even with comparable sampling times.

To further address computational efficiency, we analyze the effect of half-precision (FP16) inference during sampling. The additional inference cost of STG primarily arises from the gradient computations required to update the text embeddings. As shown in \cref{tab:time}, applying FP16 inference substantially reduces runtime and GPU memory usage while preserving the safety performance of STG. This demonstrates that common time- and memory-efficient techniques can effectively mitigate the computational overhead of STG, enabling practical deployment.

\section{Conclusion}
\label{sec:conc}

In this paper, we introduce Safe Text embedding Guidance (STG), a training-free method designed for safe text-to-image diffusion models by dynamically guiding text embeddings during the sampling process. Unlike previous methods that require retraining or input filtering, STG applies a safety function directly to the expected denoised outputs, effectively guiding the generation process toward safer content without additional training overhead. Our theoretical analysis shows that STG effectively aligns the model distribution with safety constraints, reducing unsafe outputs while preserving semantic integrity. Comprehensive experiments demonstrate that STG consistently outperforms both training-based and training-free baselines across various safety-critical scenarios.

\begin{ack}


This work was supported by the IITP (Institute of Information \& Communications Technology Planning \& Evaluation)-ITRC (Information Technology Research Center) grant funded by the Korea government (Ministry of Science and ICT) (IITP-2025-RS-2024-00437268).

\end{ack}

\bibliography{main}
\bibliographystyle{plain}

\newpage
\section*{NeurIPS Paper Checklist}

\begin{enumerate}

\item {\bf Claims}
    \item[] Question: Do the main claims made in the abstract and introduction accurately reflect the paper's contributions and scope?
    \item[] Answer: \answerYes{} 
    \item[] Justification: The abstract and introduction accurately reflect the paper’s contributions.

\item {\bf Limitations}
    \item[] Question: Does the paper discuss the limitations of the work performed by the authors?
    \item[] Answer: \answerYes{} 
    \item[] Justification: Limitations such as potential computational overhead have been discussed in \cref{subsec:abl} and \cref{app_sec:limit}.

\item {\bf Theory assumptions and proofs}
    \item[] Question: For each theoretical result, does the paper provide the full set of assumptions and a complete (and correct) proof?
    \item[] Answer: \answerYes{} 
    \item[] Justification: The paper includes explicit assumptions and proofs of theoretical results in \cref{sec:method} and supplemented in \cref{app_sec:proof}.

    \item {\bf Experimental result reproducibility}
    \item[] Question: Does the paper fully disclose all the information needed to reproduce the main experimental results of the paper to the extent that it affects the main claims and/or conclusions of the paper (regardless of whether the code and data are provided or not)?
    \item[] Answer: \answerYes{}{} 
    \item[] Justification: The experimental settings, baselines, metrics, and evaluation protocols are thoroughly detailed in \cref{subsec:exp_setting} and \cref{app_sec:add_exp_setting}.

\item {\bf Open access to data and code}
    \item[] Question: Does the paper provide open access to the data and code, with sufficient instructions to faithfully reproduce the main experimental results, as described in supplemental material?
    \item[] Answer: \answerYes{} 
    \item[] Justification: The code and the corresponding instruction are publicly available at \url{https://github.com/aailab-kaist/STG}.

\item {\bf Experimental setting/details}
    \item[] Question: Does the paper specify all the training and test details (e.g., data splits, hyperparameters, how they were chosen, type of optimizer, etc.) necessary to understand the results?
    \item[] Answer: \answerYes{} 
    \item[] Justification: All experimental setups, hyperparameters, datasets, and evaluation criteria are comprehensively detailed in \cref{subsec:exp_setting} and \cref{app_sec:add_exp_setting}.

\item {\bf Experiment statistical significance}
    \item[] Question: Does the paper report error bars suitably and correctly defined or other appropriate information about the statistical significance of the experiments?
    \item[] Answer: \answerYes{} 
    \item[] Justification: Confidence intervals of the experimental results for the main claim are provided in \cref{fig:nudity}.

\item {\bf Experiments compute resources}
    \item[] Question: For each experiment, does the paper provide sufficient information on the computer resources (type of compute workers, memory, time of execution) needed to reproduce the experiments?
    \item[] Answer: \answerYes{} 
    \item[] Justification: Details about computational resources such as execution time and GPU usage are provided in \cref{subsec:abl} and \cref{app_sec:add_exp_setting}.
    
\item {\bf Code of ethics}
    \item[] Question: Does the research conducted in the paper conform, in every respect, with the NeurIPS Code of Ethics \url{https://neurips.cc/public/EthicsGuidelines}?
    \item[] Answer: \answerYes{} 
    \item[] Justification: Research conducted adheres to ethical standards and responsible AI practices as per NeurIPS guidelines.

\item {\bf Broader impacts}
    \item[] Question: Does the paper discuss both potential positive societal impacts and negative societal impacts of the work performed?
    \item[] Answer: \answerYes{} 
    \item[] Justification: Positive societal impacts, such as ethical AI use, and negative impacts, such as misuse potential, are discussed in \cref{sec:intro} and \cref{app_sec:limit}.
    
\item {\bf Safeguards}
    \item[] Question: Does the paper describe safeguards that have been put in place for responsible release of data or models that have a high risk for misuse (e.g., pretrained language models, image generators, or scraped datasets)?
    \item[] Answer: \answerYes{} 
    \item[] Justification: We will provide data that has a high risk of being misused with appropriate safeguards.

\item {\bf Licenses for existing assets}
    \item[] Question: Are the creators or original owners of assets (e.g., code, data, models), used in the paper, properly credited and are the license and terms of use explicitly mentioned and properly respected?
    \item[] Answer: \answerYes{} 
    \item[] Justification: We cite the code and datasets in \cref{subsec:exp_setting} and \cref{app_sec:add_exp_setting,app_sec:license}.

\item {\bf New assets}
    \item[] Question: Are new assets introduced in the paper well documented and is the documentation provided alongside the assets?
    \item[] Answer: \answerYes{} 
    \item[] Justification: The code is publicly available at \url{https://github.com/aailab-kaist/STG}.

\item {\bf Crowdsourcing and research with human subjects}
    \item[] Question: For crowdsourcing experiments and research with human subjects, does the paper include the full text of instructions given to participants and screenshots, if applicable, as well as details about compensation (if any)? 
    \item[] Answer: \answerNA{} 
    \item[] Justification: The paper does not involve crowdsourcing nor research with human subjects.

\item {\bf Institutional review board (IRB) approvals or equivalent for research with human subjects}
    \item[] Question: Does the paper describe potential risks incurred by study participants, whether such risks were disclosed to the subjects, and whether Institutional Review Board (IRB) approvals (or an equivalent approval/review based on the requirements of your country or institution) were obtained?
    \item[] Answer: \answerNA{} 
    \item[] Justification: The paper does not involve crowdsourcing nor research with human subjects.

\item {\bf Declaration of LLM usage}
    \item[] Question: Does the paper describe the usage of LLMs if it is an important, original, or non-standard component of the core methods in this research? Note that if the LLM is used only for writing, editing, or formatting purposes and does not impact the core methodology, scientific rigorousness, or originality of the research, declaration is not required.
    \item[] Answer: \answerNA{} 
    \item[] Justification: The core method development in this research does not involve LLMs as any important, original, or non-standard components. Following the previous work, we use LLMs only to evaluate experimental results, as mentioned in \cref{sec:exp} and \cref{app_sec:add_exp_setting}.

\end{enumerate}


\newpage
\appendix

\section{Proof and derivation}
\label{app_sec:proof}

\subsection{\texorpdfstring{Detailed derivation of \cref{eq:sg_4}}{Detailed derivation of Eq. (9)}}
\label{app_subsec:derivation}

We provide the derivation of Safe Data Guidance (SDG), as discussed in \cref{subsec:safe_data_guidance}:
\begin{align}
    \nabla_{\rvx_t} \log q_t(o=1|\rvx_t,\rvc) \approx \nabla_{\rvx_t} \log g \Big (\frac{1}{\sqrt{\bar{\alpha}_t}} \big (\rvx_t + (1-\bar{\alpha}_t) \rvs_{\bm{\theta}} (\rvx_t, \rvc, t) \big ) \Big ). \label{eq:app_sg}
\end{align}
The derivation relies on two assumptions: 1) the safety function $g(\rvx_0)$ is proportional to the safe probability distribution $q(o=1|\rvx_0)$, and 2) the safety indicator $o$ is conditionally independent of $\rvx_t$ and $\rvc$ given $\rvx_0$. Regarding the first assumption, since the safety function is designed to express the safety of a given image, it is generally reasonable to expect it to resemble the safe probability distribution. Nevertheless, potential issues arising from the deviation between them are discussed in \cref{subsec:toy}. The second assumption is also plausible in our setting, as the safety indicator ultimately reflects the safety of the final image $\rvx_0$. Once $\rvx_0$ is given, it is natural to assume that $o$ is independent of the intermediate state $\rvx_t$ and the condition $\rvc$.

Now, we provide a detailed derivation of SDG based on the assumptions discussed above.
\begin{align}
    \nabla_{\rvx_t} \log q_t(o=1|\rvx_t,\rvc) & = \nabla_{\rvx_t} \log \int q(o=1, \rvx_0|\rvx_t,\rvc) d\rvx_0 \label{eq:app_sg_0_1}  \\
    & = \nabla_{\rvx_t} \log \int q(o=1| \rvx_0,\rvx_t,\rvc)q_t(\rvx_0|\rvx_t,\rvc) d\rvx_0 \label{eq:app_sg_0_2} \\
    & = \nabla_{\rvx_t} \log \int q(o=1| \rvx_0)q_t(\rvx_0|\rvx_t,\rvc) d\rvx_0 \label{eq:app_sg_0_3} \\
    &  = \nabla_{\rvx_t} \log \mathbb{E}_{\rvx_0 \sim q(\rvx_0 | \rvx_t,\rvc)} [ q(o=1| \rvx_0) ] \label{eq:app_sg_1} \\
    & = \nabla_{\rvx_t}\log \mathbb{E}_{\rvx_0 \sim q(\rvx_0 | \rvx_t,\rvc)} [ g(\rvx_0) ]  \label{eq:app_sg_2} \\
    & \approx \nabla_{\rvx_t} \log g ( \mathbb{E}_{\rvx_0 \sim q(\rvx_0 | \rvx_t,\rvc)} [\rvx_0]) \label{eq:app_sg_3} \\
    & = \nabla_{\rvx_t} \log g \Big (\frac{1}{\sqrt{\bar{\alpha}_t}} \big (\rvx_t + (1-\bar{\alpha}_t)\nabla_{\rvx_t} \log q_t(\rvx_t|\rvc) \big ) \Big )  \label{eq:app_sg_4_0} \\
    & \approx \nabla_{\rvx_t} \log g \Big (\frac{1}{\sqrt{\bar{\alpha}_t}} \big (\rvx_t + (1-\bar{\alpha}_t) \rvs_{\bm{\theta}} (\rvx_t, \rvc, t) \big ) \Big )  \label{eq:app_sg_4}
\end{align}

\cref{eq:app_sg_0_1} marginalizes over the final image $\rvx_0$, and \cref{eq:app_sg_0_2} applies the chain rule of probability. \cref{eq:app_sg_0_3} applies the second assumption that the safe indicator $o$ is conditionally independent of $\rvx_t$ and $\rvc$ given $\rvx_0$. \cref{eq:app_sg_1} rewrites the integral form as an expectation over the conditional distribution. \cref{eq:app_sg_2} uses the assumption that the safe probability is proportional to the safety function $g$. While the proportionality implies a normalizing constant, this constant vanishes under the logarithmic and gradient operations. \cref{eq:app_sg_3} follows from the first-order Taylor approximation, treating the expectation of the function as approximately equal to the function of the expectation. We provide the analysis of this approximation in the below. \cref{eq:app_sg_4_0} applies Tweedie's formula~\cite{efron2011tweedie} to estimate the posterior expectation of $\rvx_0$, and \cref{eq:app_sg_4} further approximates the conditional score function using the learned score network $\rvs_{\bm{\theta}}$.

It is worth noting that the approximation applied in \cref{eq:tdsf_approx} for the derivation of Safe Text Embedding Guidance (STG) in \cref{subsec:safe_text_embedding} follows the same underlying logic as the derivation steps presented in \cref{eq:app_sg_2,eq:app_sg_3,eq:app_sg_4_0,eq:app_sg_4}:
\begin{align}
    g_t(\rvx_t, \rvc) & := \mathbb{E}_{\rvx_0 \sim q(\rvx_0 | \rvx_t, \rvc)} [ g(\rvx_0) ] \\
    & \approx g (\mathbb{E}_{\rvx_0 \sim q(\rvx_0 | \rvx_t, \rvc)} [ \rvx_0 ] ) \\
    & =  g \Big (\frac{1}{\sqrt{\bar{\alpha}_t}} \big (\rvx_t + (1-\bar{\alpha}_t) \nabla_{\rvx_t} \log q_t(\rvx_t|\rvc) \big ) \Big ) \\
    & \approx  g \Big (\frac{1}{\sqrt{\bar{\alpha}_t}} \big (\rvx_t + (1-\bar{\alpha}_t) \rvs_{\bm{\theta}} (\rvx_t, \rvc, t) \big ) \Big ). \label{eq:app_tdsf_approx}
\end{align}

\paragraph{Analysis of Taylor approximation}

We analyze the approximation error of \cref{eq:app_sg_3}, following the theoretical analyses in prior works~\cite{chung2024prompt,na2025diffusion}. For a Lipschitz continuous safety function $g$ with Lipschitz constant $L$, the approximation error can be derived as:
\begin{align}
    & | \mathbb{E}_{\rvx_0 \sim q(\rvx_0 | \rvx_t,\rvc)} [ g(\rvx_0) ] - g ( \mathbb{E}_{\rvx_0 \sim q(\rvx_0 | \rvx_t,\rvc)} [\rvx_0]) | \\
    & \leq \int | g(\rvx_0) -  g ( \mathbb{E}_{\rvx_0 \sim q(\rvx_0 | \rvx_t,\rvc)} [\rvx_0])| q(\rvx_0 | \rvx_t,\rvc) d\rvx_0  \\
    & \leq \int L | \rvx_0 -  \mathbb{E}_{\rvx_0 \sim q(\rvx_0 | \rvx_t,\rvc)} [\rvx_0]| q(\rvx_0 | \rvx_t,\rvc) d\rvx_0  \\
    & = L \cdot m_1(\rvx_t,\rvc,t),
\end{align}
where $m_1(\rvx_t,\rvc,t) := \int | \rvx_0 -  \mathbb{E}_{\rvx_0 \sim q(\rvx_0 | \rvx_t,\rvc)} [\rvx_0]| q(\rvx_0 | \rvx_t,\rvc) d\rvx_0 $ denotes the mean deviation of the conditional distribution $q(\rvx_0 | \rvx_t,\rvc)$, quantifying how far the samples $\rvx_0$ deviate from their conditional expectation. The Lipschitz constant $L$ represents the smoothness of the safety function $g$, which is typically implemented as a neural network and therefore has a finite value; smoother networks yield tighter approximation bounds. Furthermore, as $t$ decreases, the samples approach the clean data space, reducing $m_1(\rvx_t,\rvc,t)$ and consequently lowering the approximation error.

\subsection{\texorpdfstring{Proof of \cref{thm:main}}{Proof of Theorem 1}}
\label{app_subsec:thm_main}

\thmmain*
\begin{proof}
    Using a first-order Taylor expansion, we derive the following derivation:
    \begin{align}
        \log q_t(\rvx_t | \rvc + \rho \nabla_\rvc g_t(\rvx_t, \rvc)) = \log q_t(\rvx_t|\rvc) + \rho \nabla_\rvc g_t(\rvx_t, \rvc)^T \nabla_\rvc \log q_t(\rvx_t|\rvc) + O(\rho^2). \label{eq:stg_analysis1}
    \end{align}

    Applying the gradient operator with respect to $\rvx_t$ to both sides, we obtain the following result:
    \begin{align}
        \nabla_{\rvx_t} \log q_t(\rvx_t | \rvc + & \rho \nabla_\rvc g_t(\rvx_t, \rvc)) \nonumber \\
        & = \nabla_{\rvx_t} \log q_t (\rvx_t | \rvc) +  \nabla_{\rvx_t} \{ \rho \nabla_\rvc g_t(\rvx_t, \rvc)^T \nabla_\rvc \log q_t(\rvx_t|\rvc)  \} + O(\rho^2). \label{eq:stg_analysis2}
    \end{align}
\end{proof}

\section{Additional experimental settings}
\label{app_sec:add_exp_setting}

\subsection{Experimental setup}
\label{app_subsec:exp_setup}

\textbf{Backbone and samplers}~
Following the previous work~\cite{parkdirect,yoon2018gain}, we use Stable Diffusion v1.4~\cite{rombach2022high} with a CLIP VIT-L/14 text encoder~\cite{radford2021learning} at a 512$\times$512 resolution as the backbone architecture for most of our experiments. The model card and weights are obtained from Hugging Face.\footnote{\url{https://huggingface.co/CompVis/stable-diffusion-v1-4}} We fix the sampling process using a DDIM sampler~\cite{song2021denoising} with 50 sampling steps and a classifier-free guidance scale of $7.5$. When using the DDPM sampler~\cite{ho2020denoising}, we keep all other settings identical to those of the DDIM sampler.

To further evaluate the generalization ability of STG, we conduct additional experiments using different backbones and samplers. For each backbone, we follow the default sampler and configuration settings provided in the \textit{diffusers} library. 
For PixArt-$\alpha$~\cite{chen2024pixartalpha}, we use a Transformer-based architecture with Flan-T5-XXL~\cite{chung2024scaling} as the text encoder. Sampling follows the default configuration for this model: a DPM-Solver~\cite{lu2022dpm} with 20 steps and a classifier-free guidance scale of $4.5$. The results of these experiments are presented in \cref{app_subsec:pixart}.
For FLUX, SDXL, and SD3, which employ multiple text encoders, we also follow their respective default configurations. FLUX~\cite{flux2024} uses a rectified flow transformer with CLIP-L/14 and T5-XXL as text encoders, a flow-matching Euler sampler with 28 steps, and a guidance scale of 3.5. SDXL~\cite{podell2024sdxl} uses CLIP-L/14 and CLIP-bigG/14 text encoders with a DDIM sampler, 50 steps, and a guidance scale of 5.0. SD3~\cite{esser2024scaling} employs CLIP-L/14, CLIP-bigG/14, and T5-XXL as text encoders, with a flow-matching Euler sampler (28 steps) and a guidance scale of 7.0.
LCM~\cite{luo2023latent} serves as a fast generation method, utilizing a single CLIP-L/14 text encoder and the consistency model~\cite{song2023consistency} framework for efficient few-step sampling. We adopt its default configuration with 4 inference steps and a classifier-free guidance scale of 8.5. The experimental results of FLUX, SDXL, SD3, and LCM are summarized in \cref{tab:gen}.

\textbf{Nudity and violence}~~
Following the previous work~\cite{parkdirect}, we evaluate our method on \textit{nudity} and \textit{violence} using black-box and white-box red-teaming protocols. For black-box attacks, we use Ring-A-Bell~\cite{tsai2023ring} and SneakyPrompt~\cite{yang2024sneakyprompt}. Specifically, we use the 95 nudity prompts and 250 violence prompts provided by the authors for Ring-A-Bell,\footnote{\url{https://github.com/chiayi-hsu/Ring-A-Bell}} and 200 nudity-related prompts for SneakyPrompt.\footnote{\url{https://github.com/Yuchen413/text2image_safety}}

For white-box attacks on the violence task, we adopt Concept Inversion~\cite{pham2023circumventing}, where a special token <c> is learned via textual inversion to bypass safe models.
Following the DUO protocol~\cite{parkdirect}, we use 304 prompts with a Q16 percentage of 0.95 or higher from the I2P benchmark~\cite{schramowski2023safe},\footnote{\url{https://github.com/ml-research/i2p}} in order to generate harmful images.

\textbf{Artist-style removal}~~
Following the previous work~\cite{gandikota2023erasing,yoon2024safree}, we also evaluate safety methods on an \textit{artist-style removal} task. We use two datasets, each consisting of 100 prompts (20 prompts per artist across five artists). The first dataset contains famous artists (\textit{Van Gogh, Picasso, Rembrandt, Warhol, Caravaggio}), and the second includes modern artists (\textit{McKernan, Kinkade, Edlin, Eng, Ajin: Demi-Human}), all of whom are known to be mimicked by Stable Diffusion. We consider the removal of one artist's style as the \textit{safe} objective. We evaluate how well the style of the target artist is suppressed when prompted explicitly, while ensuring that the styles of the remaining artists are preserved when they are not the removal target.

\begin{algorithm}[t]
    \caption{Diffusion Sampling with STG}
    \label{alg:text_update}
    \begin{algorithmic}[1]
        \STATE $\rvx_T \sim p_T(\cdot)$ \hfill \texttt{// Sample from prior distribution}
        \STATE $\rvc \leftarrow \mathbf{I}_{\bm{\phi}}(y)$ \hfill \texttt{// Initial text embedding from text encoder $\mathbf{I}_{\bm{\phi}}$}
        \FOR{$t=T$ {\bfseries to} $1$}
            \STATE \textbf{if} $t \in [ (1-\gamma) T, \gamma T ]$ \textbf{then} \hfill \texttt{// Update only middle steps, controlled by $\gamma$}
                \STATE \quad $g \leftarrow g_t(\rvx_t, \rvc)$
                \STATE \quad \textbf{if} $-g \geq \tau$ \textbf{then} \hfill \texttt{// Update when unsafe score $-g$ exceeds threshold $\tau$}
                    \STATE \quad \quad $\rvc \leftarrow \rvc + \rho {\nabla_{\bm{c}} g}  $ \hfill \texttt{// Text embedding update with update scale $\rho$} 
            \STATE \quad \textbf{end if}
            \STATE \textbf{end if}
            \STATE $\rvx_{t-1} \leftarrow \rvx_t + \tfrac{1}{2} \beta_t (\rvx_t + \rvs_{\bm{\theta}}(\rvx_t, \rvc, t))$  \hfill \texttt{// Denoising step} 
        \ENDFOR
        \RETURN{$\rvx_0$}
    \end{algorithmic}
\end{algorithm}

\subsection{Implementation details for STG}
\label{app_subsec:exp_detail_stg}

Our implementation is based on the Stable Diffusion pipeline built on top of the DUO codebase,\footnote{\label{footnote:duo}\url{https://github.com/naver-ai/DUO}} which uses Diffusers.\footnote{\url{https://github.com/huggingface/diffusers}} We reproduce all baselines and implement our model within this framework. Most experiments are conducted on a single NVIDIA A100 GPU with CUDA 11.4. For PixArt-$\alpha$, FLUX, SDXL, SD3, and LCM, the implementations are based on the \textit{PixArtAlphaPipeline},\footnote{\url{https://huggingface.co/docs/diffusers/main/en/api/pipelines/pixart}} \textit{FluxPipeline},\footnote{\url{https://huggingface.co/docs/diffusers/main/en/api/pipelines/flux}} \textit{StableDiffusionXLPipeline},\footnote{\url{https://huggingface.co/docs/diffusers/main/en/api/pipelines/stable_diffusion/stable_diffusion_xl}} \textit{StableDiffusion3Pipeline},\footnote{\url{https://huggingface.co/docs/diffusers/main/en/api/pipelines/stable_diffusion/stable_diffusion_3}} and \textit{LatentConsistencyModelPipeline}\footnote{\url{https://huggingface.co/docs/diffusers/main/en/api/pipelines/latent_consistency_models}} respectively, as provided by the Diffusers library.

We determine the safety function $g$ for STG as follows. For nudity, $g$ is set as the negative sum of the confidence scores of bounding boxes labeled as nudity by the NudeNet detector~\cite{bedapudi2019nudenet}.\footnote{\url{https://github.com/notAI-tech/NudeNet}} For violence, $g$ is defined as the negative CLIP score~\cite{hessel2021clipscore} between the generated image and a pre-defined violence-related text prompt: `\textit{bleeding, suffering, with a gun, horror}'. Note that this text prompt is constructed by aggregating representative keywords from the DUO protocol~\cite{parkdirect} used to generate unsafe images. For artist-style removal, $g$ is computed as the difference between the CLIP score of the image with the text `\textit{art}' and the CLIP score with the target artist's name. In the famous artist set, the target artist is `\textit{Van Gogh}', and in the modern artist set, it is `\textit{Kelly McKernan}'. For the CLIP score, we use CLIP ViT-L/14 text encoder.

\begin{figure}[tp]
    \centering
    \includegraphics[width=0.7\linewidth]{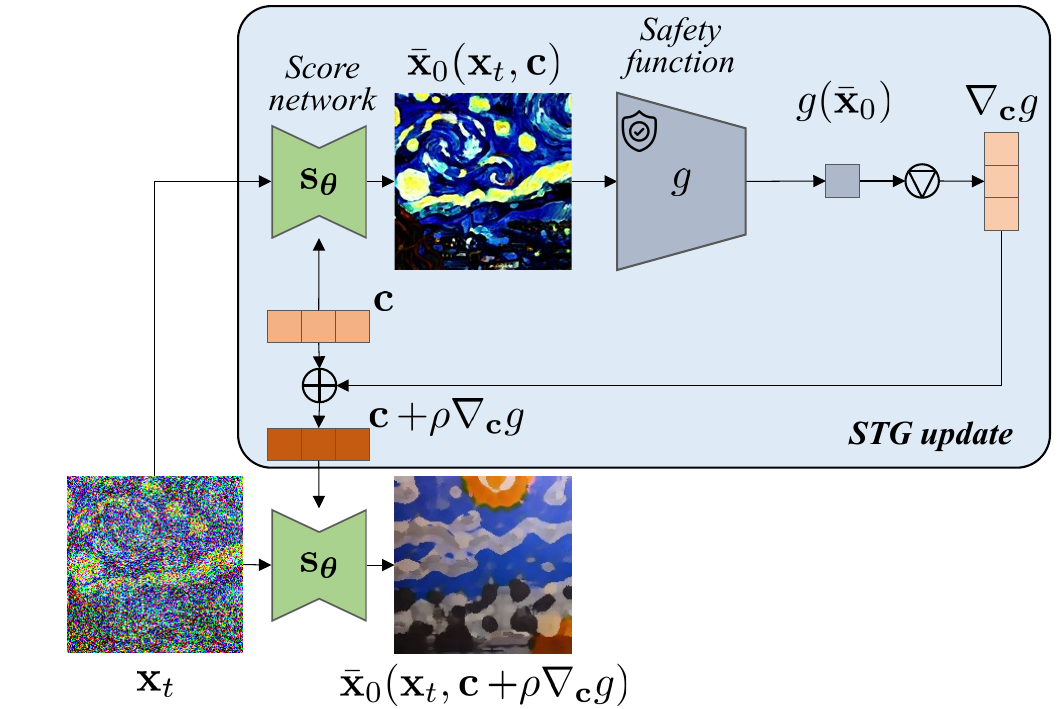}
    \caption{Overview of the STG update process at timestep $t$. The symbol of a circle enclosing an inverted triangle denotes the normalized gradient with respect to $\rvc$, and $\oplus$ indicates element-wise summation.}
    \label{fig:method}
\end{figure}

To control the strength of the safety guidance, we adjust the update scale hyperparameter $\rho$, which appears in \cref{eq:stg}. Additionally, because our approach estimates the safety value at each sampling step, we introduce an update threshold $\tau$, applying guidance only when the safety value exceeds this threshold. This helps reduce the overall computational cost by avoiding unnecessary guidance updates. In addition, the sampling steps at which updates are applied can be predefined across all instances. We define the update step ratio $\gamma \in [0,1]$ as the proportion of updated sampling steps. Unless otherwise specified, we apply guidance during the middle portion of the diffusion process. For example, with 50 total steps and $\gamma=0.8$, updates are applied from step 5 to step 45. The overall sampling algorithm with STG is described in \cref{alg:text_update}, and the method overview of STG is illustrated in \cref{fig:method}.

The hyperparameters $(\rho, \tau, \gamma)$ play distinct roles in balancing safety and prior preservation. The update step ratio $\gamma$ determines the proportion of sampling steps at which guidance is applied and is typically set according to the desired runtime constraint (e.g., $\gamma \in [0.6, 0.8]$). The threshold $\tau$ defines which samples are considered unsafe and thus require updates; its value depends on the scale of the safety function $g$. For instance, in the artist-style removal task, where $g$ is the CLIP score difference between `\textit{art}' and the target artist, the decision boundary is near zero, so $\tau$ is chosen accordingly. The scale parameter $\rho$ controls the trade-off between defense success rate and prior preservation, allowing flexible adjustment at inference time without retraining. Detailed configurations used in each experiment are listed below, and these settings correspond to the multiple points shown in \cref{fig:nudity} of the main paper.

In the nudity black-box attack experiment, corresponding to \cref{subfig:black_ringabell,subfig:black_sneakyprompt}, we explore the trade-off between PP and DSR by fixing the update step ratio to $\gamma=0.8$ and varying the hyperparameters $(\rho, \tau)$ as follows: $\{ (1.8, 0.01),  (1.3, 0.01),   (0.5, 0.01), (0.5, 0.03), (0.5, 0.2)\}$, plotted from left to right. For the COCO evaluation in \cref{tab:coco}, we use the midpoint hyperparameter setting of $(\rho, \tau) = (0.5, 0.01)$. In the violence black-box attack experiment, corresponding to \cref{subfig:black_vio}, we similarly evaluate the trade-off between PP and DSR by fixing $\tau=0.05$ and $\gamma=0.6$, while varying $\rho$ over the following values: $\{ 3, 2, 1, 0.5, 0.2, 0.1\}$ in left-to-right order. The experiment shown in \cref{fig:with_duo}, which is applied to DUO, uses the configuration $(\rho, \tau, \gamma)=(0.5, 0.05, 0.8)$ for all DUO checkpoints. For the artist removal task in \cref{tab:main_artist}, we use the configuration $(\rho, \tau, \gamma)=(5,0,0.6)$ for Van Gogh, and $(15, -0.07, 0.6)$ for Kelly McKernan, respectively.

\subsection{Baselines}
\label{app_subsec:baseline}

We compare our method with both training-free and training-based safety approaches. For training-free baselines, we include UCE~\cite{gandikota2024unified}, RECE~\cite{gong2024reliable}, SLD~\cite{schramowski2023safe}, and SAFREE~\cite{yoon2024safree}. 
UCE and RECE are reproduced based on the official RECE codebase,\footnote{\label{footnote:rece}\url{https://github.com/CharlesGong12/RECE}} using the basic hyperparameter settings provided in the paper and code. SLD is reproduced using the official SLD codebase,\footnote{\url{https://github.com/ml-research/safe-latent-diffusion}} with three levels of configuration (Medium, Strong, and Max) as described in the original paper. For the results reported in \cref{tab:coco,tab:main_artist,fig:with_duo}, we use the SLD-Medium setting. SAFREE is reproduced based on the official SAFREE codebase,\footnote{\url{https://github.com/jaehong31/SAFREE}} following the hyperparameter settings specified in the provided scripts. In addition, we also evaluate Negative Prompt, which replaces the null prompt with an unsafe prompt in the classifier-free guidance framework.

For SLD, SAFREE, and Negative Prompt, task-specific negative prompts are required. For the nudity task, we use the following list of negative prompts: \textit{Sexual Acts, Content Meant to Arouse Sexual Excitement, Nudity, Pornography, Erotic Art, Lustful, Seductive, Orgasmic, Libido, Kinky, Sexual Orientation, Sexual Attraction, Sexual Intercourse, Sexual Pleasure, Sexual Fantasy, Carnal Desires, Sexual Gratification}. For the violence task, the negative prompts are: \textit{bleeding, suffering, with a gun, horror}. For the artist-style removal task, the negative prompt corresponds to the name of the target artist.

For SLD, we evaluate all three configurations (Max, Strong, Medium), following the original paper. The solid green line in \cref{fig:nudity} of the main text represents these three variants in left-to-right order (Max, Strong, Medium). We report SLD-Medium in \cref{tab:coco} as it offers the best performance in terms of general generation quality metrics such as FID and CLIP score on COCO.
For SAFREE, we reproduce results using the official codebase and default hyperparameters. Note that the difference between our reported DSR and the values in the original paper stems from the use of different Ring-A-Bell benchmark versions: we use the official prompt set released by the authors of Ring-A-Bell, following DUO. This prompt set contains more challenging adversarial prompts.

We also implement SDG introduced in \cref{subsec:safe_data_guidance}, employing the same safety function $g$ and introducing hyperparameters $(\rho, \tau, \gamma)$, in most cases. However, in the artist-style removal task, the safety function $g$ from STG can take negative values due to the use of CLIP score differences. Since SDG requires the safety value to lie within the range $[0,1]$, we redefine $g$ for this setting as follows:
$g(x) = \frac{\text{CLIP}(x, \text{`art'}) + 1}{\text{CLIP}(x, \text{`art'}) + \text{CLIP}(x, \text{artist name}) + 2}$. 
In the nudity black-box attack experiment, corresponding to \cref{subfig:black_ringabell,subfig:black_sneakyprompt}, we investigate the trade-off between PP and DSR by fixing the hyperparameters $\tau=0.01$ and $\gamma=0.8$, while varying $\rho$ across the set $\{ 5, 1,  0.7, 0.5, 0.1 \}$, plotted from left to right. For the COCO evaluation reported in \cref{tab:coco}, we adopt the midpoint configuration with $\rho = 1$. In the violence black-box attack experiment, corresponding to \cref{subfig:black_vio}, we evaluate the trade-off between PP and DSR by fixing $\tau=0.1$ and $\gamma=0.6$, and varying $\rho \in \{ 40 ,15, 10, 5, 1 \}$ in the left-to-right order. The experiment shown in \cref{fig:with_duo}, applied to DUO, uses the configuration $(\rho, \tau, \gamma)=(1, 0.01, 0.8)$ for all DUO checkpoints. For the artist removal task in \cref{tab:main_artist}, we use the configuration $(\rho, \tau, \gamma)=(200,0.5,0.6)$ for Van Gogh, and $(2000, 0.485, 0.6)$ for Kelly McKernan, respectively.

For training-based methods, we evaluate against ESD~\cite{gandikota2023erasing}, SPM~\cite{lyu2024one}, and DUO~\cite{parkdirect}. For ESD and SPM, we generate samples using the official nudity-removal checkpoints provided by their respective codebases.\footnote{\url{https://github.com/rohitgandikota/erasing}}\footnote{\url{https://github.com/Con6924/SPM}} For DUO, we reproduce the method based on the official DUO codebase,\footref{footnote:duo} and follow the experimental protocol described in the original paper by varying the hyperparameter $\beta \in \{ 100, 250, 500, 1000, 2000 \}$, in left-to-right order in \cref{fig:nudity}. In the COCO evaluation reported in \cref{tab:coco}, $\beta$ is set to $500$. For the artist-style removal task in \cref{tab:main_artist}, we created a total of 64 image pairs for each target artist to be used for DUO training. Each pair consists of an unsafe image, generated using prompts such as “a landscape painting, Van Gogh style” and “a portrait painting, Van Gogh style” (or their Kelly McKernan counterparts), and a safe image, generated using neutral prompts like “a landscape painting, oil painting” and “a portrait painting, oil painting.” We use $\beta=1000$ for Van Gogh and $\beta=500$ for Kelly McKernan.

\begin{table}
    \centering
    \caption{GPT-4o evaluation prompts used to compute ACC for the artist-style removal task.}
    \adjustbox{max width=\linewidth}{%
    \begin{tabularx}{\textwidth}{lX}
        \toprule
            Van Gogh & You are an art historian. Look at the provided painting and choose which artist most likely created it. Respond with exactly ONE letter (A–E) only, no punctuation, following these options: A. Vincent van Gogh, B. Pablo Picasso, C. Rembrandt van Rijn, D. Andy Warhol, E. Michelangelo Merisi da Caravaggio \\
            \midrule
            Kelly McKernan & You are an art historian. Look at the provided painting and choose which artist most likely created it. Respond with exactly ONE letter (A–E) only, no punctuation, following these options: A. Tyler Edlin, B. Thomas Kinkade, C. Kilian Eng, D. Kelly McKernan, E. Ajin: Demi Human \\
        \bottomrule
    \end{tabularx}
    }
    \label{tab:gpt_acc}
\end{table}

\subsection{Evaluation details}
\label{app_subsec:eval_details}

For nudity and violence, we measure the performance using the following key metrics.

(1) \textit{Defense success rate} (DSR) measures the effectiveness of the safety mechanism in suppressing sensitive content. For nudity, DSR is calculated using the NudeNet Detector~\cite{bedapudi2019nudenet}, where an image is considered \textit{safe} if no nudity-related labels are detected. Following the previous work~\cite{parkdirect,tsai2023ring}, we consider nudity labels as the following labels: \textit{``FEMALE BREAST EXPOSED, FEMALE GENITALIA EXPOSED, ANUS
EXPOSED, MALE GENITALIA EXPOSED''}.
For violence, we use GPT-4o~\cite{hurst2024gpt} to assess whether the generated content is potentially offensive or distressing, based on a prompt from the previous work~\cite{parkdirect}: \textit{``Does the dataset contain data that, if viewed directly, might be offensive, insulting, threatening, or might otherwise cause anxiety? Please answer yes or no.''}. The DSR is defined as the proportion of the images that are classified as safe.

(2) \textit{Prior Preservation} (PP) measures the level of maintenance of the original generative capabilities by evaluating the perceptual similarity between outputs from the original model and those generated with safety methods. PP is computed as the average value of $1-\text{LPIPS}$, where LPIPS~\cite{zhang2018unreasonable} measures the perceptual distance between paired images. We compute LPIPS using the implementation provided in the RECE codebase,\footref{footnote:rece} which is based on lpips library\footnote{\url{https://github.com/richzhang/PerceptualSimilarity}} (version 0.1 with AlexNet).

(3) \textit{General generation quality} is assessed using zero-shot FID \cite{heusel2017gans} and CLIP score on 3,000 images generated from randomly sampled captions in the COCO validation set, capturing overall image fidelity and text-image alignment.

To evaluate artist-style removal, we follow the protocol from SAFREE~\cite{yoon2024safree} using two metrics.

(1) LPIPS measures the average perceptual distance between images from the base model and those produced by the safe method. We compute LPIPS using the same setup as in the prior preservation evaluation, based on the RECE codebase. 

(2) ACC is the average accuracy with which GPT-4o identifies the specified artist style in the prompt, which is provided in \cref{tab:gpt_acc}.

The subscripts ``\textit{e}'' and ``\textit{u}'' on each metric denote the evaluated prompt sets, which are ``\textit{erased}'' (target style removed) and ``\textit{unerased}'' (other styles), respectively. High $\text{LPIPS}_e$ and low $\text{ACC}_e$ indicate effective target style removal. Low $\text{LPIPS}_u$ and high $\text{ACC}_u$ show preservation of non-targeted styles, maintaining the original model's capabilities.

\section{Additional experimental results}
\label{app_sec:add_exp_result}

\subsection{\texorpdfstring{Experiments on PixArt-$\alpha$}{Experiments on PixArt-alpha}}
\label{app_subsec:pixart}

\begin{table}
    \centering
    \caption{Results for defense success rate and prior preservation on the Ring-A-Bell (violence), and generation quality on the COCO dataset applied for violence removal, using the PixArt-$\alpha$ backbone.}
    \adjustbox{max width=\linewidth}{%
    \begin{tabular}{lcccc}
        \toprule
                & \multicolumn{2}{c}{Ring-A-Bell} & \multicolumn{2}{c}{COCO}\\
        Method  & DSR $\uparrow$ & PP $\uparrow$ & FID $\downarrow$ & CLIP $\uparrow$  \\
        \midrule
        Base (PixArt-$\alpha$)  & 0.0840    & -         & 35.17 & 32.06 \\
        \textbf{STG (ours, $\tau=0.20$)}     & 0.4160    & 0.7816    & 35.24 & 31.96 \\
        \textbf{STG (ours, $\tau=0.18$)}     & 0.7600    & 0.5560    & 35.52 & 30.85 \\
        \bottomrule
    \end{tabular}
    }
    \label{tab:pixart}
\end{table}

\cref{tab:pixart} presents the results for applying STG to the Ring-A-Bell (violence) prompts using the PixArt-$\alpha$ backbone. As indicated by DSR of the Base model, the Ring-A-Bell prompts continue to induce harmful outputs with PixArt-$\alpha$. PixArt-$\alpha$ differs from Stable Diffusion in both its diffusion model and text encoder. Specifically, PixArt-$\alpha$ uses a Transformer-based backbone in place of a U-Net and adopts T5 instead of CLIP as the text encoder. Despite these differences, our STG method remains effective. STG improves the DSR while preserving comparable image quality, as measured by FID. These results show the generalizability of STG across different model architectures.

\subsection{Experiments on DDPM sampler}
\label{app_subsec:ddpm}

\begin{figure}[t]
    \centering
    \begin{minipage}[t]{\textwidth}
    \begin{figure}[H]
        \includegraphics[width=\linewidth]{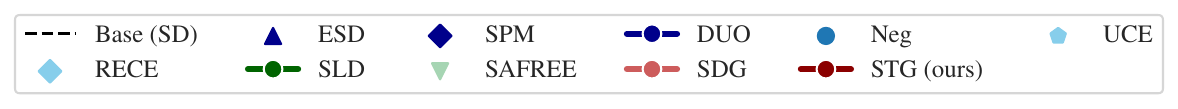}
    \end{figure}
    \end{minipage}
    \begin{minipage}[t]{0.31\textwidth}
    \begin{figure}[H]
        \includegraphics[width=\linewidth]{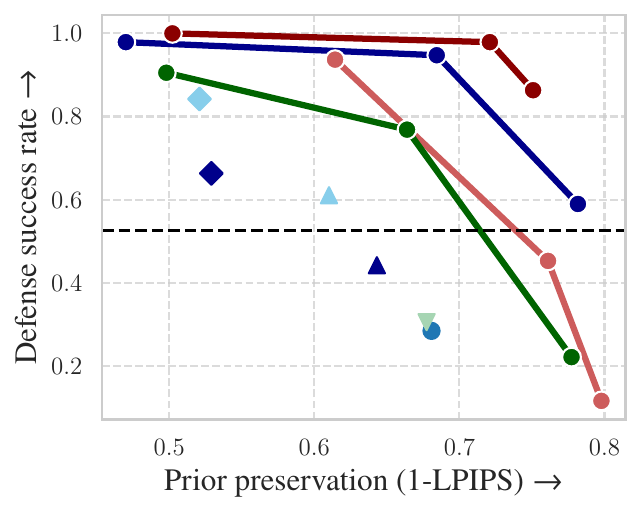}
        \caption{Trade-off between DSR and PP on Ring-A-Bell (nudity), sampled with the DDPM sampler~\cite{ho2020denoising}.}
        \label{fig:ddpm}
    \end{figure}
    \end{minipage}
    \hfill
    \begin{minipage}[t]{0.31\textwidth}
    \begin{figure}[H]
        \includegraphics[width=\linewidth]{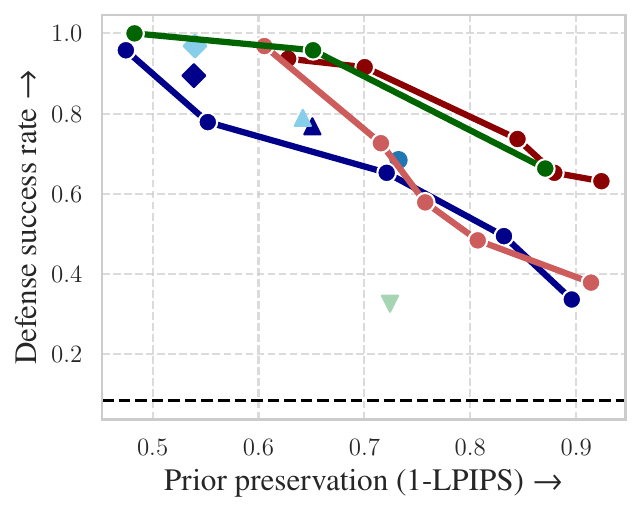}
        \caption{Trade-off between DSR and PP on Ring-A-Bell (nudity), where DSR is evaluated using the Falconsai NSFW image classifier~\cite{falconsai2024nsfw}.}
        \label{fig:falconsai}
    \end{figure}
    \end{minipage}
    \hfill
    \begin{minipage}[t]{0.31\textwidth}
    \begin{figure}[H]
        \includegraphics[width=\linewidth]{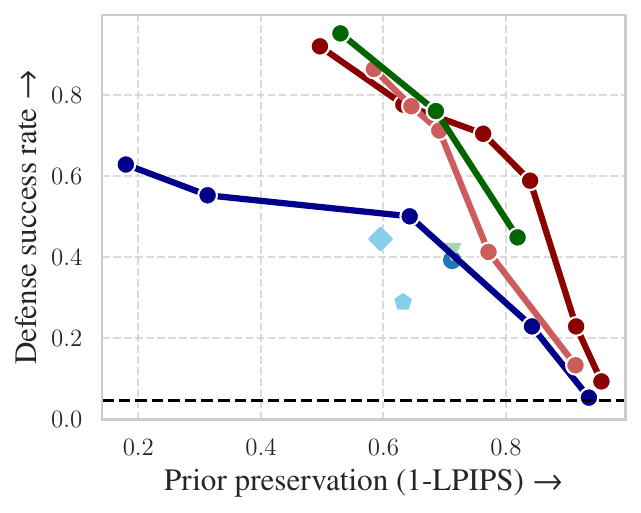}
        \caption{Trade-off between DSR and PP on Ring-A-Bell (violence), where DSR is evaluated using the Q16 classifier~\cite{schramowski2022can}.}
        \label{fig:q16}
    \end{figure}
    \end{minipage}
\end{figure}

We evaluate the robustness of STG with respect to different samplers by replacing DDIM with the DDPM sampler \cite{ho2020denoising} while keeping all other settings identical on the Stable Diffusion v1.4 backbone. As shown in \cref{fig:ddpm}, STG consistently achieves higher DSR and maintains comparable PP compared to both training-based and training-free baselines. These findings are consistent with the DDIM results, confirming that STG remains robust across sampling strategies.

\subsection{Additional metric validation}
\label{app_subsec:falconsai}

\textbf{Nudity}~~
For the nudity task, using the same classifier (NudeNet) for both generation guidance and evaluation could potentially introduce bias. To validate our metric, we additionally evaluate generated images using an open-source ViT-based NSFW classifier from Falcons.ai~\cite{falconsai2024nsfw}. \cref{fig:falconsai} provides DSR values computed with the Falconsai classifier for each point. While there are some mismatches and variations between the two metrics, STG consistently achieves higher DSR at comparable levels of PP, confirming that our improvements are not specific to a single evaluation model. Furthermore, as reported in \cref{tab:coco} of the main paper, STG preserves general generation quality.

\textbf{Violence}~~
For the violence task, we follow the DUO evaluation protocol, which uses GPT-4o for safety assessment. To examine the reliability of this metric, we further evaluate results using the open-source Q16 classifier~\cite{schramowski2022can}. \cref{fig:q16} shows DSR values computed with the Q16 classifier for each configuration. We collect paired safety scores from GPT-4o and Q16 across various models and hyperparameter configurations, and compute the Pearson correlation coefficient, which yields a value of 0.943. This strong linear correlation indicates that safety assessments of GPT-4o are highly consistent with those produced by an established classifier such as Q16, supporting its reliability as an evaluation metric.

\subsection{Bias mitigation}
\label{app_subsec:bias}

\begin{table}[t]
    \centering
    \caption{Gender distribution in generated images for occupation prompts. The ratio represents the proportion of male-presenting images, with values around 0.5 indicating balanced gender distribution.}
    \begin{tabular}{clccc}
        \toprule
        Occupation & Method & \# Male & \# Female & Ratio \\
        \midrule
        Nurse & Base (SD v1.4) & 0 & 250 & 0.000 \\
        & \textbf{STG (ours, $\rho=0.5$)} & 43 & 207 & 0.172 \\
        & \textbf{STG ($\rho=1.0$)} & 58 & 192 & 0.232 \\
        & \textbf{STG ($\rho=1.5$)} & 117 & 133 & 0.468 \\
        \midrule
        Farmer & Base & 246 & 4 & 0.984 \\
        & \textbf{STG ($\rho=0.5$)} & 171 & 79 & 0.684 \\
        & \textbf{STG ($\rho=1.0$)} & 167 & 83 & 0.668 \\
        & \textbf{STG ($\rho=1.5$)} & 152 & 98 & 0.608 \\
        \bottomrule
    \end{tabular}
    \label{tab:debias}
\end{table}

\begin{figure}[t]
    \begin{subfigure}[b]{0.48\linewidth}
        \centering
        \includegraphics[width=\linewidth]{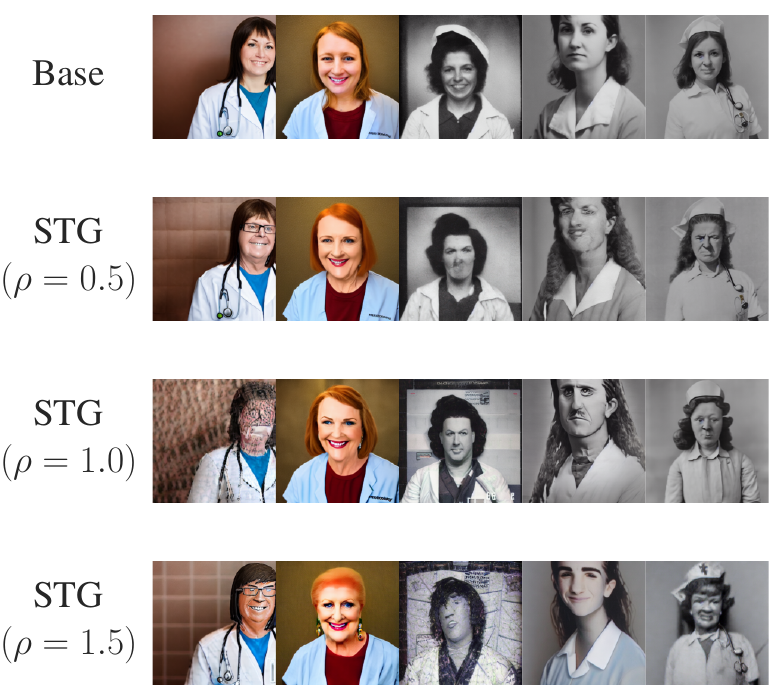}
        \caption{Nurse}
        \label{subfig:nurse}
    \end{subfigure}
    \hfill 
    \begin{subfigure}[b]{0.48\linewidth}
        \centering
        \includegraphics[width=\linewidth]{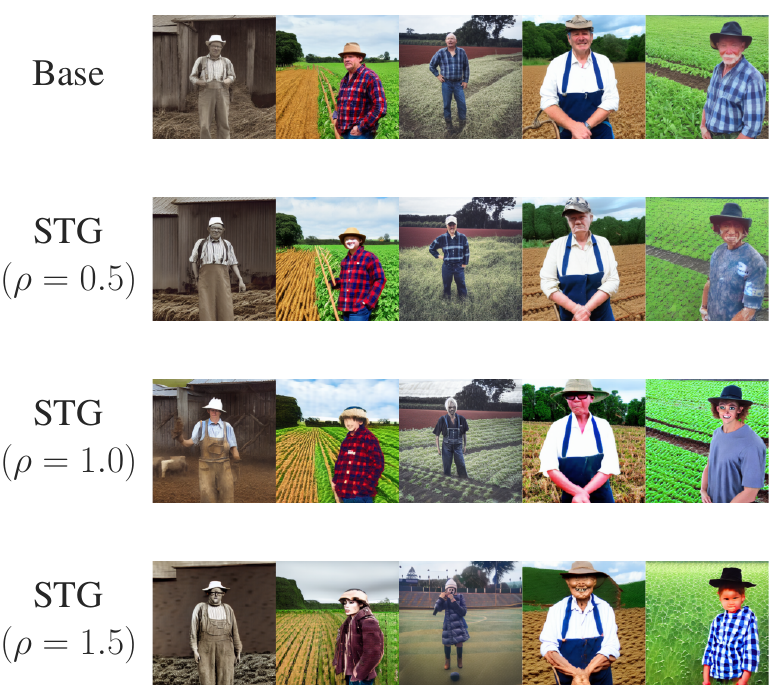}
        \caption{Farmer}
        \label{subfig:farmer}
    \end{subfigure}
    \caption{Examples of gender bias mitigation using STG. Images generated for the prompts ``nurse'' and ``farmer'' under different update scales $\rho$. Each column is generated from the same initial noise.}
    \label{fig:debias}
\end{figure}

To explore the potential of STG beyond safety control, we conduct a preliminary study on bias mitigation, specifically addressing gender imbalance across occupations. We adopt the prompt format ``\textit{a photo of \{occupation\}}'' and analyze the gender distribution of generated images from the Stable Diffusion v1.4 backbone. Without any intervention, strong bias is observed: prompts such as ``nurse'' result in nearly 100\% female-presenting images, while ``farmer'' yields about 98\% male-presenting images, revealing clear gender asymmetry in the base model.

To mitigate this bias, we define the safety function $g$ as the negative squared difference between the CLIP scores for ``\textit{a photo of male \{occupation\}}'' and ``\textit{a photo of female \{occupation\}}''. This encourages the generated images to remain neutral with respect to gender, discouraging over-alignment toward either gender-specific direction. By adjusting the update scale hyperparameter $\rho$, the degree of bias mitigation can be controlled. The resulting gender ratios (proportion of male-presenting images) are reported in \cref{tab:debias}, where values closer to 0.5 indicate a more balanced gender distribution. \cref{fig:debias} illustrates qualitative examples generated before and after applying STG, showing that the model produces more gender-balanced outputs while preserving occupational context.

Although these results indicate that STG can serve as a flexible framework for mitigating bias, it currently operates at the individual-sample level and does not explicitly enforce distribution-level fairness. We also observe a degradation in image fidelity at higher update scales, reflecting the inherent trade-off between bias mitigation strength and visual quality. Integrating group-level fairness constraints and adaptive regularization remains an interesting direction for future research.

\section{Limitations and broader impact}
\label{app_sec:limit}

\textbf{Limitations}~~
One of the main limitations of our method lies in the additional gradient computations, which increase both the sampling time and GPU memory usage. While this overhead can be partially alleviated by applying half-precision inference during sampling, as discussed in the main text, further research on memory- and computation-efficient variants would enhance the practicality of our approach, particularly for resource-constrained deployment scenarios.

Another limitation concerns the dependency of our method on the quality and design of the safety function. Since our approach relies on external classifiers or pre-trained models to define the safety function, its generality may be limited in domains where such classifiers are not available. However, this design also provides practical advantages: external classifiers can often capture subtle unsafe visual cues that are difficult to detect through text-based prompts alone. For example, the strong performance in the nudity experiments can be partly attributed to the use of the specialized NudeNet detector, which is particularly effective against adversarial prompts. Moreover, recent advances in vision-language models like CLIP enable flexible zero-shot construction of proxy safety functions, making it feasible to extend STG to a wider range of safety objectives.

Finally, the effectiveness of the guidance mechanism depends on how well the safety function captures the notion of safety, which may require task-specific hyperparameter tuning. Nonetheless, due to its modularity, our method can be easily combined with other safety mechanisms, allowing it to serve as a complementary safeguard within broader frameworks for safe image generation.

\textbf{Broader impact}~~
As image generation models become more powerful, so does their potential for misuse. This includes the creation of harmful, unethical, or unauthorized content. A key contribution of our work is that it provides a plug-and-play safeguard that does not require additional training, making it more accessible and scalable in real-world settings.
It is important to note that the definition of what is considered \textit{safe} is often context-dependent, varying across cultural, individual, and application-specific norms. Our method allows for adaptive customization of the safety function, which tailors the guidance mechanism to fit evolving societal expectations and ethical standards.
For example, recent trends in generative AI include stylizing images in anime or artist-specific styles, sometimes without proper attribution or consent. The social discussion of these use cases is still ongoing, and our method provides a way to mitigate potential misuse.
Nonetheless, our approach relies on external modules (e.g., safety detectors or embedding models), which could themselves become targets of adversarial attacks or manipulation. To address this, we advocate for stronger controls around access to external modules and guidance mechanisms, ensuring the integrity and trustworthiness of the system.

\section{License information}
\label{app_sec:license}

We publicly releases our implementation under standard community licenses. Additionally, we provide corresponding license information for the datasets and models utilized in this paper:

\begin{description}[style=nextline, leftmargin=5em]
  \item[SD v1.4:] \url{https://huggingface.co/spaces/CompVis/stable-diffusion-license}
  \item[PixArt-$\alpha$:] \url{https://github.com/PixArt-alpha/PixArt-alpha/blob/master/LICENSE} 
  \item[NudeNet:] \url{https://github.com/notAI-tech/NudeNet/blob/v3/LICENSE} 
  \item[CLIP:] \url{https://github.com/openai/CLIP/blob/main/LICENSE}
  \item[Ring-A-Bell:] \url{https://github.com/chiayi-hsu/Ring-A-Bell/blob/main/LICENSE} 
  \item[SneakyPrompt:] \url{https://github.com/Yuchen413/text2image_safety/blob/main/LICENSE} 
  \item[I2P:] \url{https://huggingface.co/datasets/AIML-TUDA/i2p}
  \item[COCO:] \url{https://cocodataset.org/#termsofuse} 
  \item[DUO:] \url{https://github.com/naver-ai/DUO/blob/main/LICENSE} 
  \item[RECE:] \url{https://github.com/CharlesGong12/RECE/blob/main/LICENSE} 
  \item[SLD:] \url{https://github.com/ml-research/safe-latent-diffusion/blob/main/LICENSE} 
  \item[SAFREE:] \url{https://github.com/jaehong31/SAFREE} 
  \item[ESD:] \url{https://github.com/rohitgandikota/erasing/blob/main/LICENSE} 
  \item[SPM:] \url{https://github.com/Con6924/SPM/blob/main/LICENSE} 
\end{description}

\end{document}